%% file: increasing-action-gap.tex
\documentclass[letterpaper]{article}
\pdfoutput=1

\usepackage{aaai16}
\usepackage{times}
\usepackage{helvet}
\usepackage{courier}
\usepackage{graphicx}
\usepackage{color}

\usepackage{relsize}
\usepackage{stmaryrd}
\usepackage{amsmath}
\usepackage{mathrsfs}
\usepackage{amsfonts}
\usepackage{amsthm}

\newcommand{\cTc}{\cT_{\textsc{c}}}
\newcommand{\cTa}{\cT_{\textsc{a}}}
\newcommand{\sQ}{\mathcal{Q}}
\newcommand{\sV}{\mathcal{V}}

\newcommand{\expects}{\expect\nolimits}

\newcommand{\cA}{\mathcal{A}}

\newcommand{\cT}{\mathcal{T}}
\newcommand{\cX}{\mathcal{X}}

\newcommand{\cZ}{\mathcal{Z}}

\newcommand{\bR}{\mathbb{R}}

\newcommand{\bN}{\mathbb{N}}

\newcommand{\indic}[1]{\mathbb{I}_{\left [ #1 \right ]}}

\DeclareMathOperator*{\expect}{{\mathlarger {\mathbf E}}}

\newcommand{\infnorm}[1]{\left \| #1 \right \|_\infty}

\newtheorem{defn}{Definition}
\newtheorem{lem}{Lemma}
\newtheorem{rem}{Remark}
\newtheorem{cor}{Corollary}
\newtheorem{thm}{Theorem}

\newcommand{\cbar}{\, | \,}

\newcommand{\ct}[2][]{
    \ifthenelse{\isempty{#1}}{\cti{#2}}{\ctii{#2}{#1}}}
\newcommand{\cti}[1]{N_n(#1)}
\newcommand{\ctii}[2]{N(#1,#2)}

\DeclareMathOperator*{\argmax}{arg\,max}

\def \tiQ {\tilde Q}
\def \tiV {\tilde V}
\def \cTcq {\cT_{\textsc{cqvi}}}
\def \cTq {\cT_{\textsc{qvi}}}
\def \oneem {\hspace{1em}}

\def \cTal {\cT_{\textsc{al}}}
\def \cTpal {\cT_{\textsc{pal}}}
\def \tipi {\tilde \pi}
\def \tia {\tilde a}
\def \Rmax {\infnorm{R}}
\newcommand{\gamename}[1]{\textsc{#1}}

\newcommand{\citet}[1]{\citeauthor{#1} \shortcite{#1}}
\newcommand{\citep}[1]{(\citeauthor{#1} \citeyear{#1})}
\newcommand{\citenp}[1]{\citeauthor{#1} \citeyear{#1}}

\makeatletter
\newcommand{\footnoteref}[1]{\protected@xdef\@thefnmark{\ref{#1}}\@footnotemark}
\makeatother

\begin{document}

\title{Increasing the Action Gap:\\ New Operators for Reinforcement Learning}
\author{{\Large Marc G. Bellemare \and Georg Ostrovski \and Arthur Guez} \\
{\Large {\bf Philip S. Thomas\thanks{Now at Carnegie Mellon University.}} \and {\bf R\'emi Munos}} \\
Google DeepMind \\
\{bellemare,ostrovski,aguez,munos\}@google.com; philipt@cs.cmu.edu
}

\maketitle

\begin{abstract}
\begin{quote}
This paper introduces new optimality-preserving operators on Q-functions.
We first describe an operator for tabular representations, the \emph{consistent Bellman operator}, which incorporates a notion of local policy consistency. 
We show that this local consistency leads to an increase in the action gap at each state; increasing this gap, we argue, mitigates the undesirable effects of approximation and estimation errors on the induced greedy policies.
This operator can also be applied to discretized continuous space and time problems, and we provide empirical results evidencing superior performance in this context. 
Extending the idea of a locally consistent operator, we then derive sufficient conditions for an operator to preserve optimality, leading to a family of operators which includes our consistent Bellman operator. As corollaries we provide a proof of optimality for Baird's advantage learning algorithm and derive other gap-increasing operators with interesting properties. 
We conclude with an empirical study on 60 Atari 2600 games illustrating the strong potential of these new operators. 
\end{quote}
\end{abstract}

Value-based reinforcement learning is an attractive solution to planning problems in environments with unknown, unstructured dynamics. 
In its canonical form, value-based reinforcement learning produces successive refinements of an initial value function through repeated application of a convergent operator.
In particular, value iteration \citep{bellman57dynamic} directly computes the value function through the iterated evaluation of Bellman's equation, either exactly or from samples (e.g. Q-Learning, \citenp{watkins89learning}). 

In its simplest form, value iteration begins with an initial value function $V_0$ and successively computes $V_{k+1} := \cT V_k$, where $\cT$ is the Bellman operator. When the environment dynamics are unknown, $V_k$ is typically replaced by $Q_k$, the state-action value function, and $\cT$ is approximated by an empirical Bellman operator. The fixed point of the Bellman operator, $Q^*$, is the optimal state-action value function or \emph{optimal Q-function}, from which an optimal policy $\pi^*$ can be recovered.

In this paper we argue that the optimal $Q$-function is \emph{inconsistent}, in the sense that for any action $a$ which is suboptimal in state $x$, Bellman's equation for $Q^*(x,a)$ describes the value of a \emph{nonstationary} policy: upon returning to $x$, this policy selects $\pi^*(x)$ rather than $a$.
While preserving global consistency appears impractical, we propose a simple modification to the Bellman operator which provides us a with a first-order solution to the inconsistency problem. Accordingly, we call our new operator the \emph{consistent Bellman operator}. 

We show that the consistent Bellman operator generally devalues suboptimal actions but preserves the set of optimal policies. As a result, the action gap -- the value difference between optimal and second best actions -- increases. Increasing the action gap is advantageous in the presence of approximation or estimation error \citep{farahmand11actiongap}, and may be crucial for systems operating at a fine time scale such as video games \cite{togelius09super,bellemare13arcade}, real-time markets \cite{jiang15optimal}, and robotic platforms \cite{riedmiller09reinforcement,hoburg09system,deisenroth11pilco,sutton11horde}.
In fact, the idea of devaluating suboptimal actions underpins Baird's advantage learning \citep{baird99reinforcement}, designed for continuous time control, and occurs naturally when considering the discretized solution of continuous time and space MDPs (e.g. \citenp{munos98barycentric}; \citeyear{munos02variable}), whose limit is the Hamilton-Jacobi-Bellman equation \citep{kushner01numerical}. Our empirical results on the bicycle domain \citep{randlov98learning} show a marked increase in performance from using the consistent Bellman operator. 

In the second half of this paper we derive novel sufficient conditions for an operator to preserve optimality. The relative weakness of these new conditions reveal that it is possible to deviate significantly from the Bellman operator without sacrificing optimality: an optimality-preserving operator needs not be contractive, nor even guarantee convergence of the Q-values for suboptimal actions. 
While numerous alternatives to the Bellman operator have been put forward (e.g. recently \citenp{azar11speedy}; \citenp{bertsekas12qlearning}), we believe our work to be the first to propose such a major departure from the canonical fixed-point condition required from an optimality-preserving operator. As proof of the richness of this new operator family we describe a few practical instantiations with unique properties.

We use our operators to obtain state-of-the-art empirical results on the Arcade Learning Environment \citep{bellemare13arcade}. We consider the Deep Q-Network (DQN) architecture of \citet{mnih15human}, replacing only its learning rule with one of our operators. Remarkably, this one-line change produces agents that significantly outperform the original DQN.  
Our work, we believe, demonstrates the potential impact of rethinking the core components of value-based reinforcement learning.

\section{Background}

We consider a Markov decision process $M := (\cX, \cA, P, R, \gamma)$ where $\cX$ is the state space, $\cA$ is the finite action space, $P$ is the transition probability kernel, $R$ is the reward function mapping state-action pairs to a bounded subset of $\bR$, and $\gamma \in [0,1)$ is the discount factor. 
We denote by $\sQ := \sQ_{\cX, \cA}$ and $\sV := \sV_{\cX}$ the space of bounded real-valued functions over $\cX \times \cA$ and $\cX$, respectively. 
For $Q \in \sQ$ we write $V(x) := \max_a Q(x,a)$, and follow this convention for related quantities ($\tiV$ for $\tiQ$, $V'$ for $Q'$, etc.) whenever convenient and unambiguous.
In the context of a specific $(x,a) \in \cX \times \cA$ we further write $\expects_P := \expects_{x' \sim P(\cdot \cbar x, a)}$ to mean the expectation with respect to $P(\cdot \cbar x,a)$, with the convention that $x'$ always denotes the next state random variable. 

A deterministic policy $\pi : \cX \to \cA$ induces a Q-function $Q^\pi \in \sQ$ whose \emph{Bellman equation} is 
\begin{equation*}\label{eqn:bellman_equation_for_policy}
Q^\pi(x,a) := R(x,a) + \gamma \expects_P Q^\pi (x',\pi(x')) .
\end{equation*}
The state-conditional \emph{expected return} $V^\pi(x) := Q^\pi(x,\pi(x))$ is the expected discounted total reward received from starting in $x$ and following $\pi$.

The Bellman operator $\cT : \sQ \to \sQ$ is defined pointwise as
\begin{equation}\label{eqn:bellman_operator}
\cT Q(x,a) := R(x,a) + \gamma \expects_P \max_{b \in \cA} Q(x',b) .
\end{equation}
$\cT$ is a contraction mapping in supremum norm \citep{bertsekas96neurodynamic} whose unique fixed point is the optimal Q-function
\begin{equation*}
Q^*(x,a) = R(x,a) + \gamma \expects_P \max_{b \in \cA} Q^*(x',b),
\end{equation*}
which induces the optimal policy $\pi^*$: 
\begin{equation*}
\pi^*(x) := \argmax_{a \in \cA} Q^*(x,a) \quad \forall x \in \cX .
\end{equation*}

A Q-function $Q \in \sQ$ induces a greedy policy $\pi(x) := \argmax_a Q(x, a)$, with the property that $Q^{\pi} = Q$ if and only if $Q = Q^*$. For $x \in \cX$ we call $\pi(x)$ the \emph{greedy action} with respect to $Q$ and $a \ne \pi(x)$ a \emph{nongreedy action}; for $\pi^*$ these are the usual optimal and suboptimal actions, respectively. 

We emphasize that while we focus on the Bellman operator, our results easily extend to its variations such as SARSA \citep{rummery94online}, policy evaluation \citep{sutton88learning}, and fitted Q-iteration \citep{ernst05treebased}. In particular, our new operators all have a sample-based form, i.e., an analogue to the Q-Learning rule of \citet{watkins89learning}.

\section{The Consistent Bellman Operator}

It is well known (and implicit in our notation) that the optimal policy $\pi^*$ for $M$ is \emph{stationary} (i.e., time-independent) and deterministic. In looking for $\pi^*$, we may therefore restrict our search to the space $\Pi$ of stationary deterministic policies. Interestingly, as we now show the Bellman operator on $\sQ$ is \emph{not}, in a sense, restricted to $\Pi$. 

\begin{figure}
\centering{
\includegraphics[width=3in]{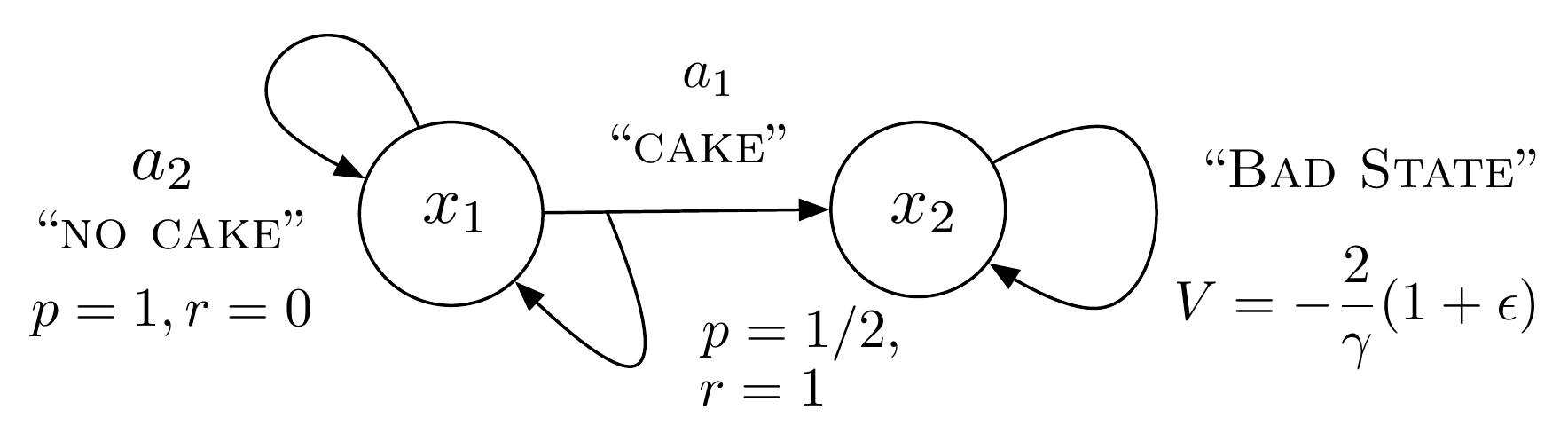}
}
\caption{A two-state MDP illustrating the non-stationary aspect of the Bellman operator. Here, $p$ and $r$ indicate transition probabilities and rewards, respectively. In state $x_1$ the agent may either eat cake to receive a reward of 1 and transition to $x_2$ with probability $\tfrac{1}{2}$, or abstain for no reward. State $x_2$ is a low-value absorbing state with $\epsilon > 0$.\label{fig:cookie_gadget}}
\end{figure}
To begin, consider the two-state MDP depicted in Figure \ref{fig:cookie_gadget}. 
This MDP abstracts a Faustian situation in which an agent repeatedly chooses between an immediately rewarding but ultimately harmful option ($a_1$), or an unrewarding alternative ($a_2$). For concreteness, we imagine the agent as faced with an endless supply of delicious cake (with $\gamma > 0$) and call these the ``cake'' and ``no cake'' actions. 

Eating cake can cause a transition to $x_2$, the ``bad state'', whose value is independent of the agent's policy: 
\begin{equation*}
V^\pi(x_2) := -2 (1 + \epsilon) \frac{1}{\gamma} \qquad \forall \pi \in \Pi.
\end{equation*}
In state $x_1$, however, the Q-values depend on the agent's future behaviour. For a policy $\pi \in \Pi$, the value of $a_1$ is 
\begin{align}
Q^\pi(x_1, a_1) &= 1 + \gamma \left [\frac{1}{2} V^\pi(x_1) + \frac{1}{2} V^\pi(x_2) \right ] \label{eqn:the_cost_of_cookie} \\
&= 1 + \frac{\gamma}{2} V^\pi(x_1) - (1 + \epsilon) = \frac{\gamma}{2} V^\pi(x_1) - \epsilon .\nonumber
\end{align}
By contrast, the value of $a_2$ is 
\begin{equation*}
Q^\pi(x_1, a_2) = 0 + \gamma V^\pi(x_1) , 
\end{equation*}
which is greater than $Q^\pi(x_1, a_1)$ for all $\pi$. It follows that not eating cake is optimal, and thus $V^*(x_1) = Q^*(x_1, a_2) = 0$. Furthermore, (\ref{eqn:the_cost_of_cookie}) tells us that the value difference between optimal and second best action, or \emph{action gap}, is 
\begin{equation*}
Q^*(x_1, a_2) - Q^*(x_1, a_1) = \epsilon .
\end{equation*}
Notice that $Q^*(x_1, a_1) = -\epsilon$ does not describe the value of any stationary policy. That is, the policy $\tipi$ with $\tipi(x_1) = a_1$ has value
\begin{equation}\label{eqn:bad_policy}
V^{\tipi}(x_1) = -\epsilon + \frac{\gamma}{2} V^{\tipi}(x_1) = \frac{-\epsilon}{1 - \gamma/2},
\end{equation}
and in particular this value is lower than $Q^*(x_1, a_1)$. Instead, $Q^*(x_1, a_1)$ describes the value of a \emph{nonstationary} policy which eats cake once, but then subsequently abstains.

So far we have considered the Q-functions of given stationary policies $\pi$, and argued that these are nonstationary. We now make a similar statement about the Bellman operator: for any $Q \in \sQ$, the nongreedy components of $Q' := \cT Q$ do not generally describe the expected return of stationary policies. Hence the Bellman operator is not restricted to $\Pi$.

When the MDP of interest can be solved exactly, this nonstationarity is a non-issue since only the Q-values for optimal actions matter. In the presence of estimation or approximation error, however, small perturbations in the Q-function may result in erroneously identifying the optimal action. 
Our example illustrates this effect: an estimate $\hat Q$ of $Q^*$ which is off by $\epsilon$ can induce a pessimal greedy policy (i.e. $\tipi$). 

To address this issue, we may be tempted to define a new Q-function which explicitly incorporates stationarity: 
\begin{align}
Q_{\textsc{stat}}^\pi(x,a) &:= R(x,a) + \gamma \expects_P \max_{b \in \cA} Q_{\textsc{stat}}^{\pi'}(x', b), \label{eqn:what_we_would_like} \\
\pi'(y) &:= \left \{ \begin{array}{ll}
    a & \text{if } y = x, \\
    \pi(y) & \text{otherwise.}
    \end{array} \right . \nonumber 
\end{align}
Under this new definition, the action gap of the optimal policy is $\frac{\epsilon}{1-\gamma/2} > Q^*(x_1, a_2) - Q^*(x_1, a_1)$. Unfortunately, (\ref{eqn:what_we_would_like}) does not visibly yield a useful operator on $\sQ$. 
As a practical approximation we now propose the \emph{consistent Bellman operator}, which preserves a local form of stationarity:
\begin{align}
\cTc Q(x,a) &:= R(x,a) \oneem + \label{eqn:consistent_bellman_operator} \\
& \gamma \expects_P \big [ \indic{x \ne x'} \max_{b \in \cA} Q(x',b) + \indic{x = x'} Q(x,a) \big ] . \nonumber
\end{align}
Effectively, our operator redefines the meaning of Q-values: if from state $x \in \cX$ an action $a$ is taken and the next state is $x' = x$ then $a$ is again taken. In our example, this new Q-value describes the expected return for repeatedly eating cake until a transition to the unpleasant state $x_2$.

Since the optimal policy $\pi^*$ is stationary, we may intuit that iterated application of this new operator also yields $\pi^*$. In fact, below we show that the consistent Bellman operator is both \emph{optimality-preserving} and, in the presence of direct loops in the corresponding transition graph, \emph{gap-increasing}:
\begin{defn}
An operator $\cT'$ is \emph{optimality-preserving} if, for any $Q_0 \in \sQ$ and $x \in \cX$, letting $Q_{k+1} := \cT' Q_k$, 
\begin{equation*}
\tiV(x) := \lim_{k \to \infty} \max_{a \in \cA} Q_k(x,a)
\end{equation*}
exists, is unique, $\tiV(x) = V^*(x)$, and for all $a \in \cA$,
\begin{equation*}
Q^*(x,a) < V^*(x,a) \hspace{-0.3em} \implies \hspace{-0.3em} \limsup_{k \to \infty} Q_k(x,a) < V^*(x). 
\end{equation*}
\end{defn}
Thus under an optimality-preserving operator at least one optimal action remains optimal, and suboptimal actions remain suboptimal. 
\begin{defn}
Let $M$ be an MDP. An operator $\cT'$ for $M$ is \emph{gap-increasing} if for all $Q_0 \in \sQ$, $x \in \cX, a \in \cA$, letting $Q_{k+1} := \cT' Q_k$ and $V_k(x) := \max_b Q_k(x,b)$, 
\begin{equation}\label{eqn:gap_increasing}
\liminf_{k \to \infty} \big [ V_k(x) - Q_k(x,a) \big ] \ge V^*(x) - Q^*(x,a) .
\end{equation}
\end{defn}
We are particularly interested in operators which are \emph{strictly gap-increasing}, in the sense that (\ref{eqn:gap_increasing}) is a strict inequality for at least one $(x,a)$ pair.

Our two-state MDP illustrates the first benefit of increasing the action gap: a greater \emph{robustness to estimation error}. Indeed, under our new operator the optimal Q-value of eating cake becomes 
\begin{equation*}
\tiQ(x_1, a_1) = \frac{\gamma}{2} \tiQ(x_1, a_1) - \epsilon = \frac{-\epsilon}{1-\gamma/2} , 
\end{equation*}
which is, again, smaller than $Q^*(x_1, a_1)$ whenever $\gamma > 0$. In the presence of approximation error in the Q-values, we may thus expect $\tiQ(x_1, a_1) < \tiQ (x_1, a_2)$ to occur more frequently than the converse.

\subsection{Aggregation Methods}

At first glance, the use of an indicator function in (\ref{eqn:consistent_bellman_operator}) may seem limiting: $P(x \cbar x, a)$ may be zero or close to zero every\-where, or the state may be described by features which preclude a meaningful identity test $\indic{x = x'}$. There is, however, one important family of value functions which have ``tabular-like'' properties: \emph{aggregation schemes} \citep{bertsekas11approximate}. As we now show, the consistent Bellman operator is well-defined for all aggregation schemes. 

An aggregation scheme for $M$ is a tuple $(\cZ,A,D)$ where $\cZ$ is a set of aggregate states, $A$ is a mapping from $\cX$ to distributions over $\cZ$, and $D$ is a mapping from $\cZ$ to distributions over $\cX$. 
For $z \in \cZ, x' \in \cX$ let $\expects_{D} := \expect_{x \sim D(\cdot \cbar z)}$ and $\expects_{A} := \expect_{z' \sim A(\cdot \cbar x')}$, where as before we assign specific roles to $x, x' \in \cX$ and $z, z' \in \cZ$. We define the \emph{aggregation Bellman operator} $\cTa : \sQ_{\cZ,\cA} \to \sQ_{\cZ,\cA}$ as
\begin{equation}\label{eqn:aggregate_bellman_operator}
\small{
\cTa Q(z,a) := \expects_{D} \left [ R(x,a) + \gamma \expects_P \expects_{A} \max_{b \in \cA} Q(z',b) \right ].
}
\end{equation}
When $\cZ$ is a finite subset of $\cX$ and $D$ corresponds to the identity transition function, i.e.~$D(x \cbar z) = \indic{x = z}$, we recover the class of averagers (\citenp{gordon95stable}; e.g., multilinear interpolation, illustrated in Figure \ref{fig:multilinear_interpolation})
and kernel-based methods \citep{ormoneit02kernelbased}. If $A$ also corresponds to the identity and $\cX$ is finite, $\cTa$ reduces to the Bellman operator (\ref{eqn:bellman_operator}) and we recover the familiar tabular representation \citep{sutton98reinforcement}.
\begin{figure}
\centering{
\includegraphics[width=2.4in]{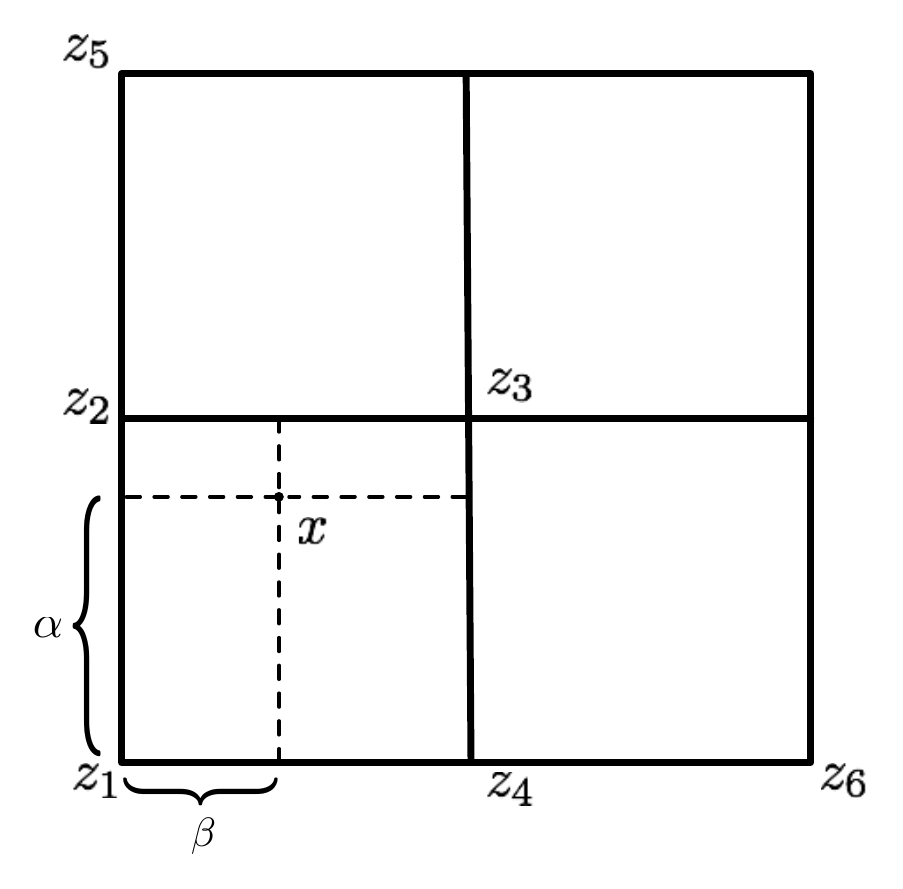}
}
\vspace{-1.5em}
\caption{Multilinear interpolation in two dimensions. The value at $x$ is approximated as $V(x) := \expects_{z' \sim A(\cdot \cbar x)} V(z')$. Here $A(z_1 \cbar x) = (1 - \alpha) (1 - \beta)$, $A(z_2 \cbar x) = \alpha (1 - \beta)$, etc.\label{fig:multilinear_interpolation}}
\end{figure}

Generalizing (\ref{eqn:consistent_bellman_operator}), we define the consistent Bellman operator $\cTc$ over $\sQ_{\cZ, \cA}$:
\begin{align}
\cTc Q(z,a) &:= \expects_{D} \Big [ R(x,a) \oneem + \label{eqn:consistent_bellman_operator_aggregation} \\
&\hspace{-2em} \gamma \expects_P \expects_A \big [ \indic{z \ne z'} \max_{b \in \cA} Q(z',b) + \indic{z = z'} Q(z,a) \big ] \Big ] . \nonumber
\end{align}
Intuitively (see, e.g., \citenp{bertsekas11approximate}), the $D$ and $A$ mappings induce a new MDP, $M' := (\cZ, \cA, P', R', \gamma)$ with
\begin{align*}
R'(z,a) & := \expects_{D} R(x,a), \\
P'(z'' \cbar z,a) & := \expects_{D} \expects_P \expects_{A} \indic{z'' = z'} .
\end{align*}
In this light, we see that our original definition of $\cTc$ and (\ref{eqn:consistent_bellman_operator_aggregation}) only differ in their interpretation of the transition kernel. Thus the consistent Bellman operator remains relevant in cases where $P$ is a deterministic transition kernel, for example when applying multilinear or barycentric interpolation to continuous space MDPs (e.g. \citenp{munos98barycentric}). 

\subsection{Q-Value Interpolation}\label{sec:q_value_interpolation}

Aggregation schemes as defined above do not immediately yield a Q-function over $\cX$. Indeed, the Q-value at an arbitrary $x \in \cX$ is defined (in the ordinary Bellman operator sense) as
\begin{equation}
Q(x,a) := R(x,a) + \gamma \expects_P \expects_{A} \max_{b \in \cA} Q(z', b), \label{eqn:q_value_for_v_interpolation}
\end{equation}
which may only be computed from a full or partial model of the MDP, or by inverting $D$. It is often the case that neither is feasible. One solution is instead to perform Q-value interpolation: 
\begin{equation*}
Q(x,a) := \expects_{z' \sim A(\cdot \cbar x)} Q(z', a),
\end{equation*}
which is reasonable when $A(\cdot \cbar x)$ are interpolation coefficients\footnote{One then typically, but not always, takes $D$ to be the identity.}. This gives the related Bellman operator
\begin{equation*}
\cTq Q(z,a) := \expects_D \left [ R(x,a) + \gamma \expects_P \max_{b \in \cA} Q(x',b) \right ],
\end{equation*}
with $\cTq Q(z,a) \le \cT Q(z,a)$ by convexity of the $\max$ operation. From here one may be tempted to define the corresponding consistent operator as
\begin{align*}
\cTq' Q(z,a) &:= \expects_D \Big [ R(x,a) \oneem + \\
&\hspace{-3.5em} \gamma \expects_P \max_{b \in \cA} \big [ Q(x',b) - A(z \cbar x') \big (Q(z, b) - Q(z, a) \big) \big ] \Big ].
\end{align*}
While $\cTq'$ remains a contraction, $\cTq' Q(z,a) \le \cTq Q(z,a)$ is not guaranteed, and it is easy to show that $\cTq'$ is not optimality-preserving. Instead we define the \emph{consistent Q-value interpolation Bellman operator} as 
\begin{equation}\label{eqn:consistent_q_averaging_operator}
\cTcq Q := \min \left \{ \cTq Q, \cTq' Q \right \} .
\end{equation}
As a corollary to Theorem \ref{thm:optimality_achieving_operators} below we will prove that $\cTcq$ is also optimality-preserving and gap-increasing. 

\subsection{Experiments on the Bicycle Domain}\label{subsec:bicycle_domain}

We now study the behaviour of our new operators on the bicycle domain \citep{randlov98learning}. In this domain, the agent must simultaneously balance a simulated bicycle and drive it to a goal 1km north of its initial position. Each time step consists of a hundredth of a second, with a successful episode typically lasting 50,000 or more steps. The driving aspect of this problem is particularly challenging for value-based methods, since each step contributes little to an eventual success and the ``curse of dimensionality'' \citep{bellman57dynamic} precludes a fine representation of the state-space. 
In this setting our consistent operator provides significantly improved performance and stability.

We approximated value functions using multilinear interpolation on a uniform $10 \times \dots \times 10$ grid over a 6-dimensional feature vector $\varphi := (\omega, \dot \omega, \theta, \dot \theta, \psi, d)$. The first four components of $\varphi$ describe relevant angles and angular velocities, while $\psi$ and $d$ are polar coordinates describing the bicycle's position relative to the goal.
We approximated Q-functions using Q-value interpolation ($\cTcq$) over this grid, since in a typical setting we may not have access to a forward model.

We are interested here in the \emph{quality} of the value functions produced by different operators. We thus computed our Q-functions using value iteration, rather than a trajectory-based method such as Q-Learning. More precisely, at each iteration we simultaneously apply our operator to all grid points, with expected next state values estimated from samples. The interested reader may find full experimental details and videos in the appendix.\footnote{Videos: https://youtu.be/0pUFjNuom1A}

While the limiting value functions ($\tiV$ and $V^*$) coincide on $\cZ \subseteq \cX$ (by the optimality-preserving property), they may differ significantly elsewhere. For $x \in \cX$ we have 
\begin{align*}
\tiV(x) &= \max\nolimits_a \tiQ(x,a) = \max\nolimits_a \expects_{A(\cdot \cbar x)} \tiQ(z, a) \\
&\ne V^*(x) 
\end{align*}
in general. This is especially relevant in the relatively high-dimensional bicycle domain, where a fine discretization of the state space is not practical and most of the trajectories take place ``far'' from grid points. As an example, consider $\psi$, the relative angle to the goal: each grid cell covers an arc of $2\pi / 10 = \pi / 5$, while a single time step typically changes $\psi$ by less than $\pi / 1000$. 

\begin{figure}
\centering{
\includegraphics[width=1.6in]{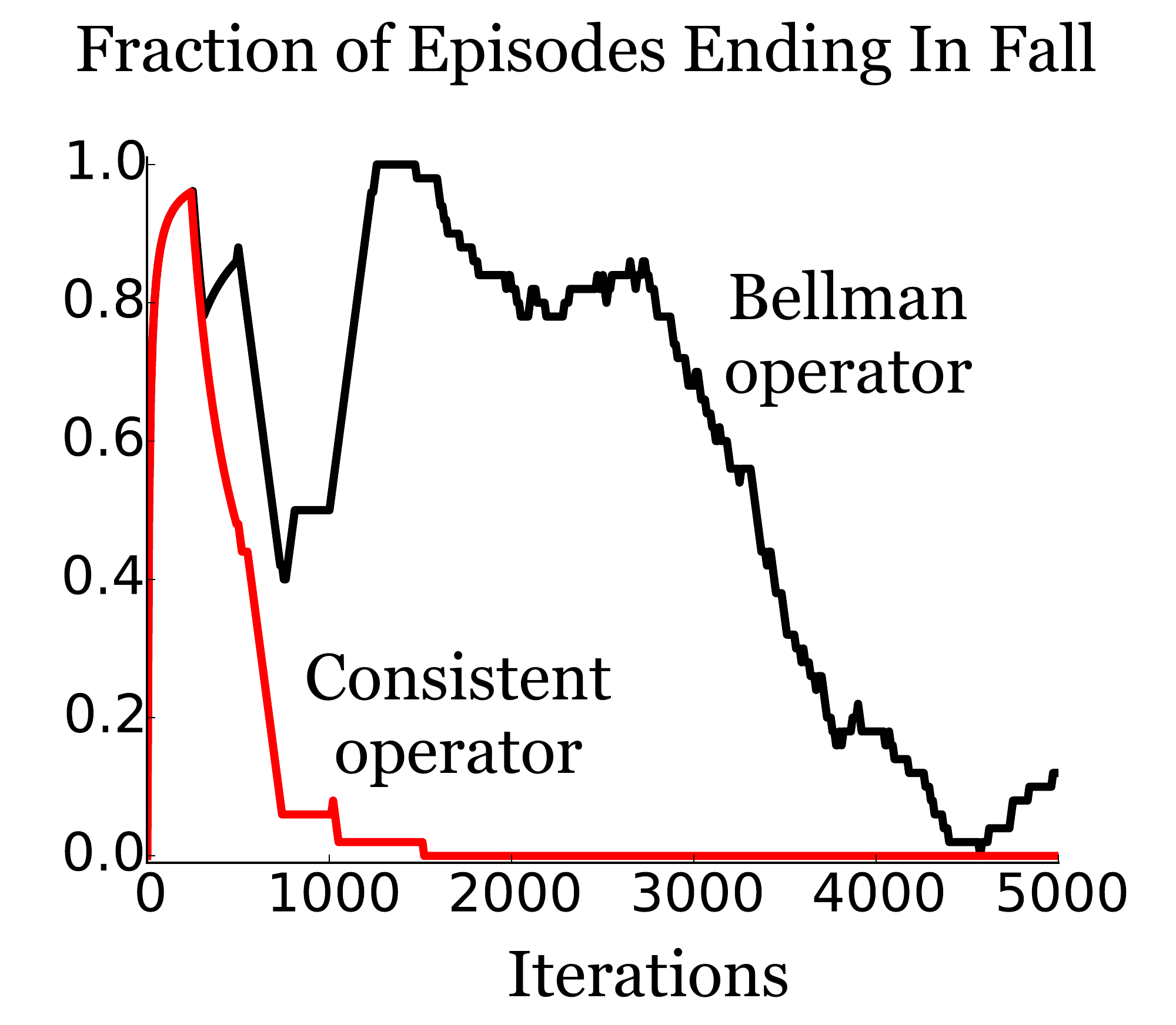}
\includegraphics[width=1.6in]{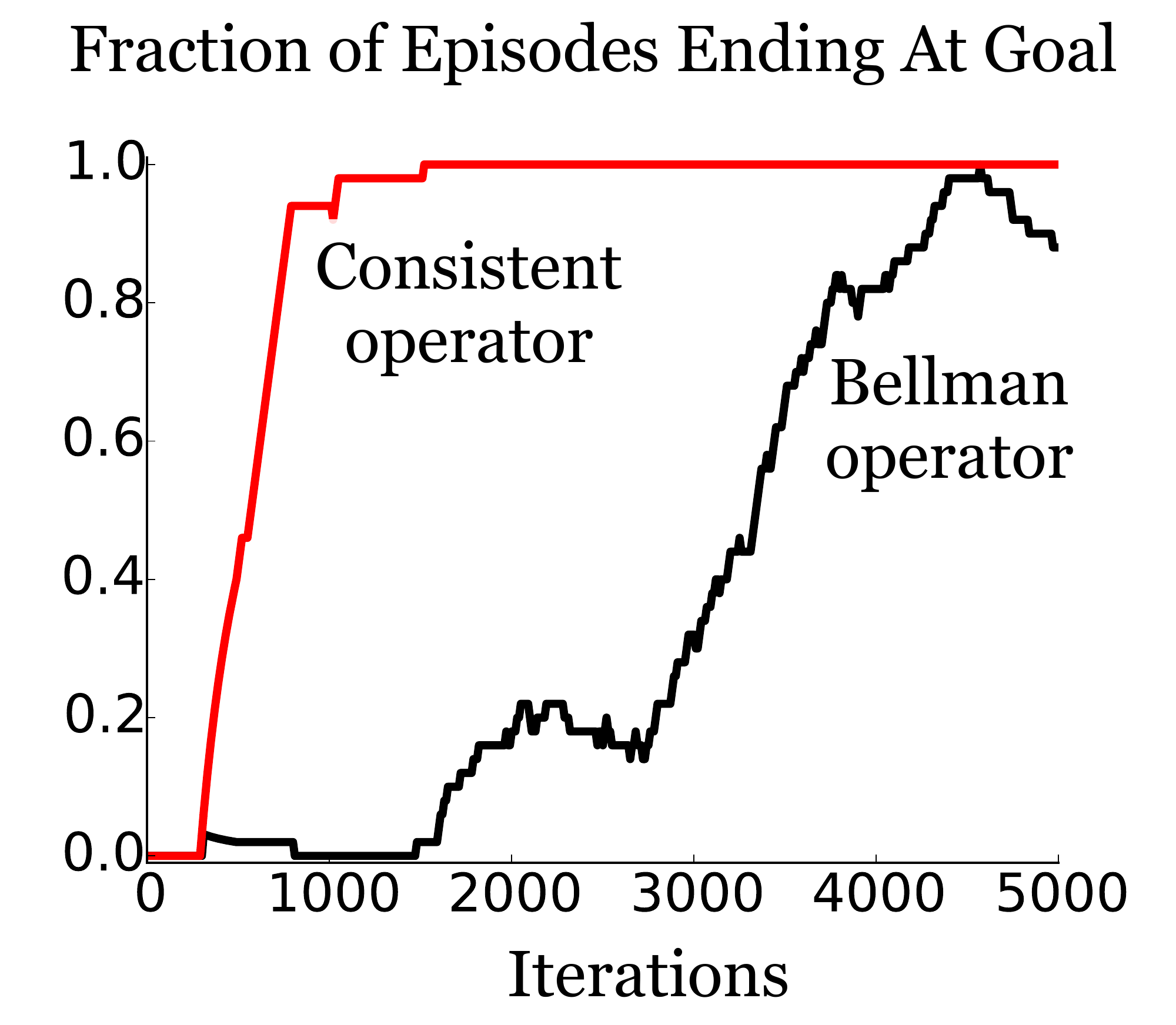}
\includegraphics[width=2.8in]{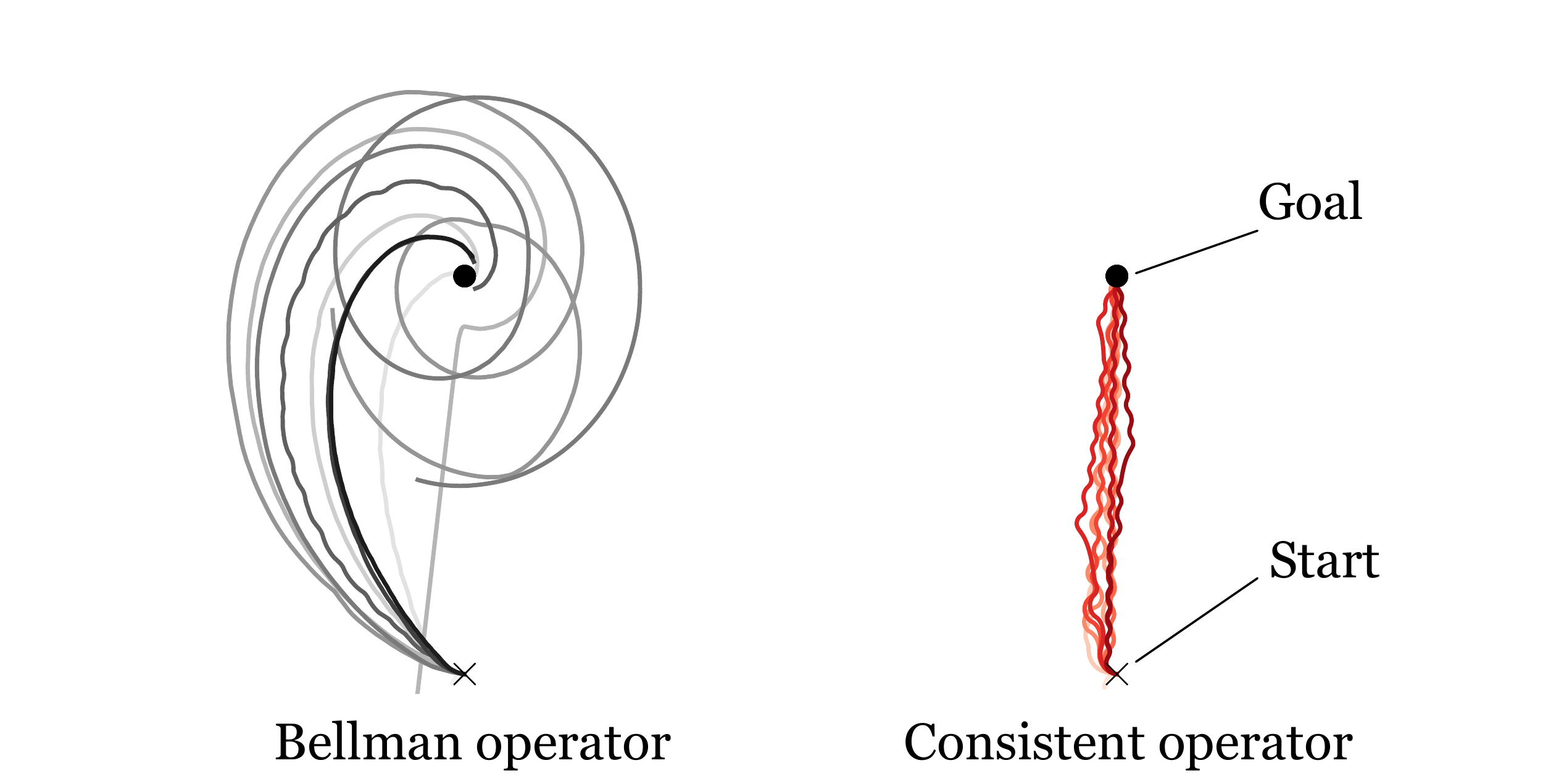}
}
\caption{\textbf{Top}. Falling and goal-reaching frequency for greedy policies derived from value iteration. \textbf{Bottom}. Sample bicycle trajectories after $100, 200, \dots, 1000$ iterations. In this coarse-resolution regime, the Bellman operator initially yields policies which circle the goal forever, while the consistent operator quickly yields successful trajectories.\label{fig:bicycle-q-averaging-performance}}
\end{figure}

Figure \ref{fig:bicycle-q-averaging-performance} summarizes our results. Policies derived from our consistent operator can safely balance the bicycle earlier on, and also reach the goal earlier than policies derived from the Bellman operator. Note, in particular, the striking difference in the trajectories followed by the resulting policies. The effect is even more pronounced when using a $8 \times \dots \times 8$ grid (results provided in the appendix).
Effectively, by decreasing suboptimal Q-values at grid points we produce much better policies \emph{within} the grid cells. This phenomenon is consistent with the theoretical results of \citet{farahmand11actiongap} relating the size of action gaps to the quality of derived greedy policies. Thus we find a second benefit to increasing the action gap: it \emph{improves policies derived from Q-value interpolation}.

\section{A Family of Convergent Operators}\label{sec:family_convergent_operators}

One may ask whether it is possible to extend the consistent Bellman operator to Q-value approximation schemes which lack a probabilistic interpretation, such as linear approximation \citep{sutton96generalization}, locally weighted regression \citep{atkeson91using}, neural networks \citep{tesauro95temporal}, or even information-theoretic methods \citep{veness15compress}. In this section we answer by the affirmative.

The family of operators which we describe here are applicable to arbitrary Q-value approximation schemes.
While these operators are in general no longer contractions, they are gap-increasing, and optimality-preserving when the Q-function is represented exactly. 
Theorem \ref{thm:optimality_achieving_operators} is our main result; one corollary is a convergence proof for Baird's advantage learning \citep{baird99reinforcement}. 
Incidentally, our taking the minimum in (\ref{eqn:consistent_q_averaging_operator}) was in fact no accident, but rather a simple application of this theorem. 

\begin{thm}\label{thm:optimality_achieving_operators}
Let $\cT$ be the Bellman operator defined by (\ref{eqn:bellman_operator}). Let $\cT'$ be an operator with the property that there exists an $\alpha \in [0, 1)$ such that for all $Q \in \sQ$, $x \in \cX, a \in \cA$, and letting $V(x) := \max\nolimits_b Q(x,b)$,
\begin{enumerate}
\item{$\cT' Q(x, a) \le \cT Q(x,a)$, and}
\item{$\cT' Q(x, a) \ge \cT Q(x,a) - \alpha \left [ V(x) - Q(x,a) \right ]$.}
\end{enumerate}
Then $\cT'$ is both optimality-preserving and gap-increasing.
\end{thm}
Thus any operator which satisfies the conditions of Theorem \ref{thm:optimality_achieving_operators} will eventually yield an optimal greedy policy, assuming an exact representation of the Q-function. 
Condition 2, in particular, states that we may subtract up to (but not including) $\max\nolimits_b Q_k(x,b) - Q_k(x,a)$ from $Q_k(x,a)$ at each iteration. This is exactly the action gap at $(x,a)$, but for $Q_k$, rather than the optimal $Q^*$. For a particular $x$, this implies we may initially devalue the optimal action $a^* := \pi^*(x)$ in favour of the greedy action. But our theorem shows that $a^*$ cannot be undervalued infinitely often, and in fact $Q_k(x, a^*)$ must ultimately reach $V^*(x)$.\footnote{When two or more actions are optimal, we are only guaranteed that \emph{one of them} will ultimately be correctly valued. The ``1-lazy'' operator described below exemplifies this possibility.} The proof of this perhaps surprising result may be found in the appendix. 

To the best of our knowledge, Theorem \ref{thm:optimality_achieving_operators} is the first result to show the convergence of iterates of dynamic programming-like operators without resorting to a contraction argument. Indeed, the conditions of Theorem 1 are particularly weak: we do not require $\cT'$ to be a contraction, nor do we assume the existence of a fixed point (in the Q-function space $\sQ$) of $\cT'$. In fact, the conditions laid out in Theorem \ref{thm:optimality_achieving_operators} characterize the set of optimality-preserving operators on $\sQ$, in the following sense:
\begin{rem}
There exists a single-state MDP $M$ and an operator $\cT'$ with either 
\begin{enumerate}
\item{$\cT' Q(x,a) > \cT Q(x,a)$ or}
\item{$\cT' Q(x,a) < \cT Q(x,a) - \left [ V(x) - Q(x,a) \right ]$,}
\end{enumerate}
and in both cases there exists a $Q_0 \in \sQ$ for which $\lim_{k \to \infty} \max_a (\cT')^k Q_0(x, a) \ne V^*(x)$.
\end{rem}
We note that the above remark does not cover the case where condition (2) is an equality (i.e., $\alpha = 1$). We leave as an open problem the existence of a divergent example for $\alpha = 1$. 

\begin{cor}
The consistent Bellman operator $\cTc$ (\ref{eqn:consistent_bellman_operator_aggregation}) and consistent Q-value interpolation Bellman operator $\cTcq$ (\ref{eqn:consistent_q_averaging_operator}) are optimality-preserving. 
\end{cor}
In fact, it is not hard to show that the consistent Bellman operator (\ref{eqn:aggregate_bellman_operator}) is a contraction, and thus enjoys even stronger convergence guarantees than those provided by Theorem \ref{thm:optimality_achieving_operators}.
Informally, whenever Condition 2 of the theorem is strengthened to an inequality, we may also expect our operators to be gap-increasing; this is in fact the case for both of our consistent operators.  

To conclude this section, we describe a few operators which satisfy the conditions of Theorem \ref{thm:optimality_achieving_operators}, and are thus optimality-preserving and gap-increasing.
Critically, none of these operators are contractions; one of them, the ``lazy'' operator, also possesses multiple fixed points.

\subsection{Baird's Advantage Learning}

The method of advantage learning was proposed by \citet{baird99reinforcement} as a means of increasing the gap between the optimal and suboptimal actions in the context of residual algorithms applied to continuous time problems.\footnote{Advantage \emph{updating}, also by Baird, is a popular but different idea where an agent maintains both $V$ and $A := Q - V$.} The corresponding operator is
\begin{align*}
\cT' Q(x,a) &= K^{-1} \big [ R(x,a) \oneem + \\
& \gamma^{\Delta_t} \expects_P V(x') + ( K - 1 ) V(x) \big ],
\end{align*}
where $\Delta_t > 0$ is a time constant and $K := C \Delta_t$ with $C > 0$. Taking $\Delta_t = 1$ and $\alpha := 1 - K$, we define a new operator with the same fixed point but a now-familiar form:
\begin{equation*}
\cTal Q(x,a) := \cT Q(x,a) - \alpha \left [ V(x) - Q(x,a) \right ].
\end{equation*}
Note that, while the two operators are motivated by the same principle and share the same fixed point, they are not isomorphic. We believe our version to be more stable in practice, as it avoids the multiplication by the $K^{-1}$ term.
\begin{cor}
For $\alpha \in [0, 1)$, the advantage learning operator $\cTal$ has a unique limit $V_{\textsc{al}} \in \sV$, and $V_{\textsc{al}} = V^*$. 
\end{cor}
While our consistent Bellman operator originates from different principles, there is in fact a close relationship between it and the advantage learning operator. Indeed,
we can rewrite (\ref{eqn:consistent_bellman_operator}) as
\begin{equation*}
\cTc Q(x,a) = \cT Q(x,a) - \gamma P(x \cbar x, a) \left [ V(x) - Q(x,a) \right ] ,
\end{equation*}
which corresponds to advantage learning with a $(x,a)$-dependent $\alpha$ parameter.

\subsection{Persistent Advantage Learning}

In domains with a high temporal resolution, it may be advantageous to encourage greedy policies which infrequently switch between actions --- to encourage a form of \emph{persistence}. 
We define an operator which favours repeated actions:
{\small
\begin{equation*}
\cTpal Q(x,a) := \max \left \{ \cTal Q(x,a), R(x,a) + \gamma \expects_P Q(x', a) \right \} .
\end{equation*}
}
Note that the second term of the $\max$ can also be written as
\begin{equation*}
\cT Q(x,a) - \gamma \expects_P \left [ V(x') - Q(x',a) \right ].
\end{equation*}
As we shall see below, \emph{persistent advantage learning} achieves excellent performance on Atari 2600 games.

\subsection{The Lazy Operator} 

As a curiosity, consider the following operator with $\alpha \in [0, 1)$: 
\begin{equation*}
\cT' Q(x,a) := \left \{ \begin{array}{lll}
    \hspace{-1pt} Q(x,a) & \text{if } Q(x,a) \hspace{-0.8em} &\le \hspace{-1pt} \cT Q(x,a) \text{ and} \\
    \hspace{-1pt} & \cT Q(x,a) \hspace{-0.8em} &\le \hspace{-1pt} \alpha V(x) \oneem + \\
    & &(1 - \alpha) Q(x,a),\\
    \hspace{-1pt} \cT Q(x,a) & \multicolumn{2}{l}{\text{otherwise.}}
    \end{array} \right .
\end{equation*}
This $\alpha$-lazy operator only updates $Q$-values when this would affect the greedy policy. And yet, Theorem \ref{thm:optimality_achieving_operators} applies! Hence $\cT'$ is optimality-preserving and gap-increasing, even though it may possess a multitude of fixed points in $\sQ$. Of note, while Theorem \ref{thm:optimality_achieving_operators} does not apply to the $1$-lazy operator, the latter is also optimality-preserving; in this case, however, we are only guaranteed that one optimal action remain optimal.  

\section{Experimental Results on Atari 2600}

We evaluated our new operators on the Arcade Learning Environment (ALE; \citenp{bellemare13arcade}), a reinforcement learning interface to Atari 2600 games. In the ALE, a frame lasts $1/60^{th}$ of a second, with actions typically selected every four frames. Intuitively, the ALE setting is related to continuous domains such as the bicycle domain studied above, in the sense that each individual action has little effect on the game.

For our evaluation, we trained agents based on the Deep Q-Network (DQN) architecture of \citet{mnih15human}. DQN acts according to an $\epsilon$-greedy policy over a learned neural-network Q-function. DQN uses an experience replay mechanism to train this Q-function, performing gradient descent on the sample squared error $\Delta_Q(x,a)^2$, where 
\begin{equation*}
\Delta Q(x,a) := R(x,a) + \gamma V(x') - Q(x,a), 
\end{equation*}
where $(x, a, x')$ is a previously observed transition. We define the corresponding errors for our operators as 
\begin{align*}
\Delta_{\textsc{al}} Q(x,a) &:= \Delta Q(x,a) - \alpha[V(x) - Q(x,a)], \\ 
\Delta_{\textsc{pal}} Q(x,a) &:= \max \Big \{ \Delta_{\textsc{al}} Q(x,a), \\
& \qquad \qquad \Delta Q(x,a) - \alpha[V(x') - Q(x',a)] \Big \} ,
\end{align*}
where we further parametrized the weight given to $Q(x',a)$ in persistent advantage learning (compare with $\cT_{\textsc{pal}}$).

Our first experiment used one of the new ALE standard versions, which we call here the \emph{Stochastic Minimal} setting. This setting includes stochasticity applied to the Atari 2600 controls, no death information, and a per-game minimal action set. 
Specifically, at each frame (not time step) the environment \emph{accepts} the agent's action with probability $1-p$, or \emph{rejects} it with probability $p$ (here, $p = 0.25$). If an action is rejected, the previous frame's action is repeated. In our setting the agent selects a new action every four frames: the stochastic controls therefore approximate a form of reaction delay.
As evidenced by a lower DQN performance, Stochastic Minimal is more challenging than previous settings. 

We trained each agent for 100 million frames using either regular Bellman updates, advantage learning (A.L.), or persistent advantage learning (P.A.L.). We optimized the $\alpha$ parameters over 5 training games and tested our algorithms on 55 more games using 10 independent trials each.

For each game, we performed a paired $t$-test (99\% C.I.) on the post-training evaluation scores obtained by our algorithms and DQN. A.L. and P.A.L. are statistically better than DQN on \textbf{37} and \textbf{35} out of 60 games, respectively; both perform worse on \textbf{one} (\textsc{Atlantis}, \textsc{James Bond}). P.A.L. often achieves higher scores than A.L., and is statistically better on \textbf{16} games and worse on \textbf{6}. These results are especially remarkable given that the only difference between DQN and our operators is a simple modification to the update rule. 

For comparison, we also trained agents using the \emph{Original DQN} setting \citep{mnih15human}, in particular using a longer 200 million frames of training. Figure \ref{fig:atari_learning_curves} depicts learning curves for two games, \textsc{Asterix} and \textsc{Space Invaders}. These curves are representative of our results, rather than exceptional: on most games, advantage learning outperforms Bellman updates, and persistent advantage learning further improves on this result. Across games, the median score improvement over DQN is \textbf{8.4\%} for A.L. and \textbf{9.1\%} for P.A.L., while the average score improvement is respectively \textbf{27.0\%} and \textbf{32.5\%}. Full experimental details are provided in the appendix.
\begin{figure}
\centering{
\includegraphics[width=1.6in]{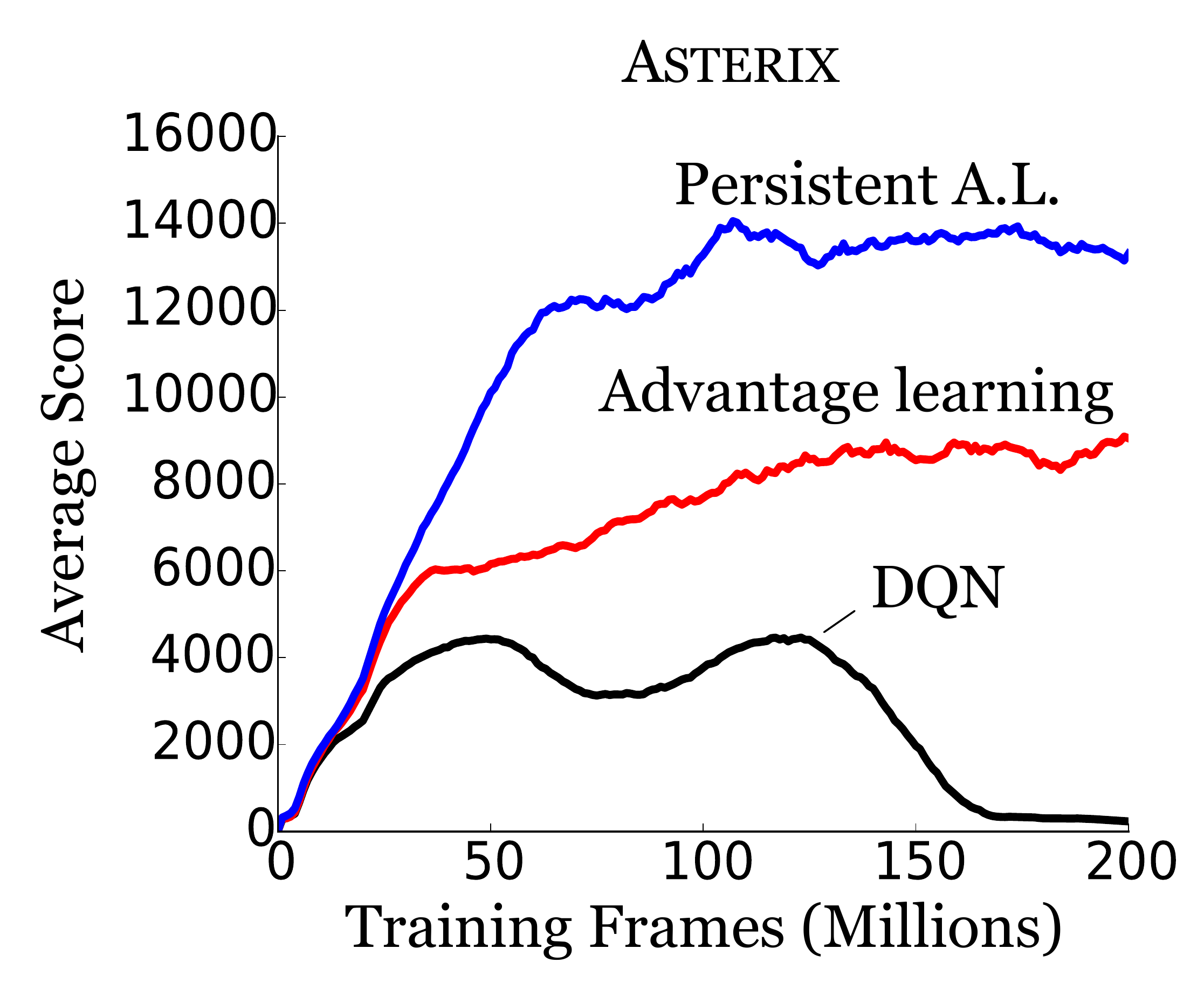}
\includegraphics[width=1.6in]{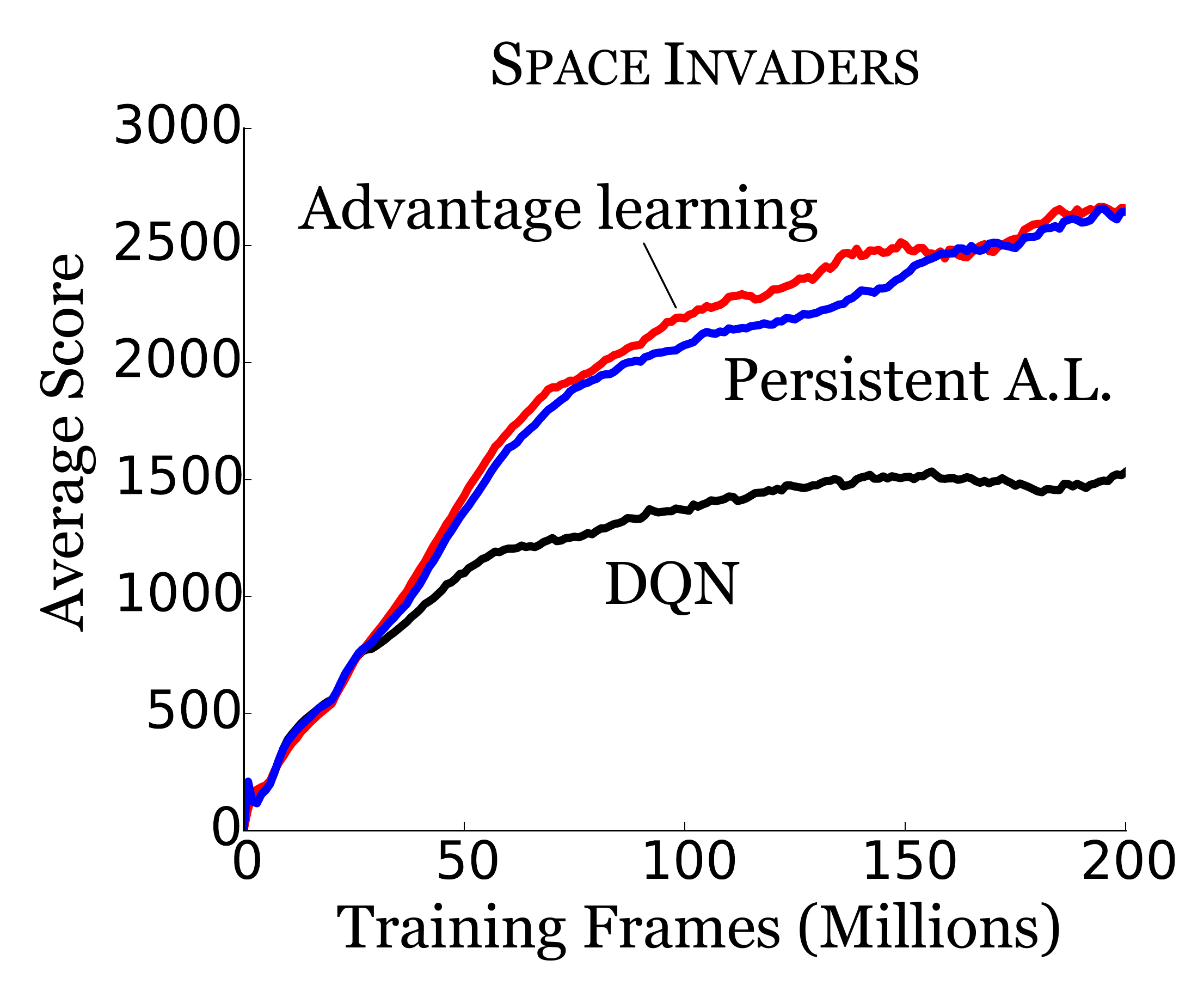}
}
\caption{Learning curves for two Atari 2600 games in the Original DQN setting.\label{fig:atari_learning_curves}}
\end{figure}

The learning curve for \textsc{Asterix} illustrates the poor performance of DQN on certain games. Recently, \citet{vanhasselt15deep} argued that this poor performance stems from the instability of the Q-functions learned from Bellman updates, and provided conclusive empirical evidence to this effect. In the spirit of their work, we compared our learned Q-functions on a single trajectory generated by a trained DQN agent playing \textsc{Space Invaders} in the Original DQN setting. For each Q-function and each state $x$ along the trajectory, we computed $V(x)$ as well as the action gap at $x$.
\begin{figure}
\centering{
\includegraphics[width=1.6in]{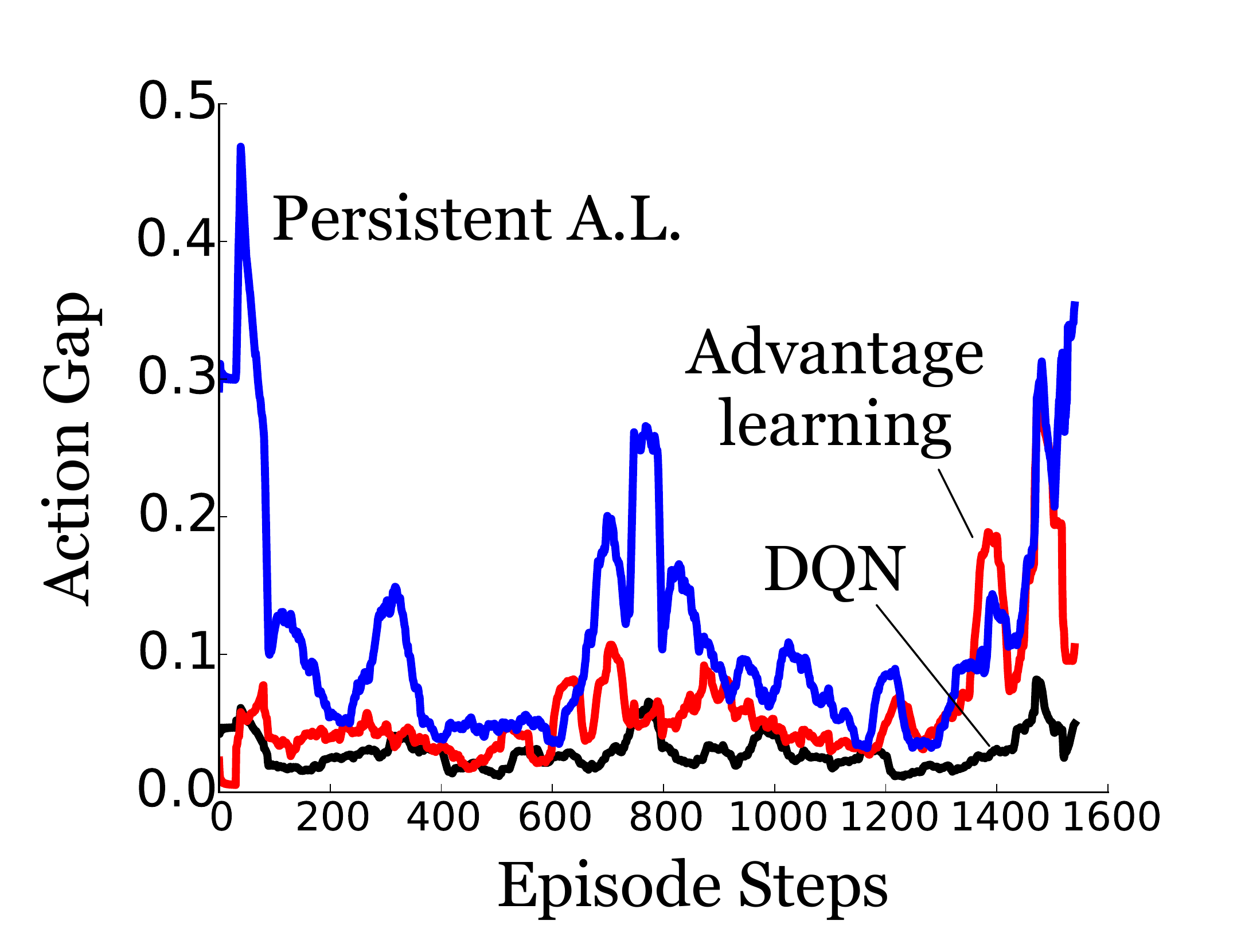}
\includegraphics[width=1.6in]{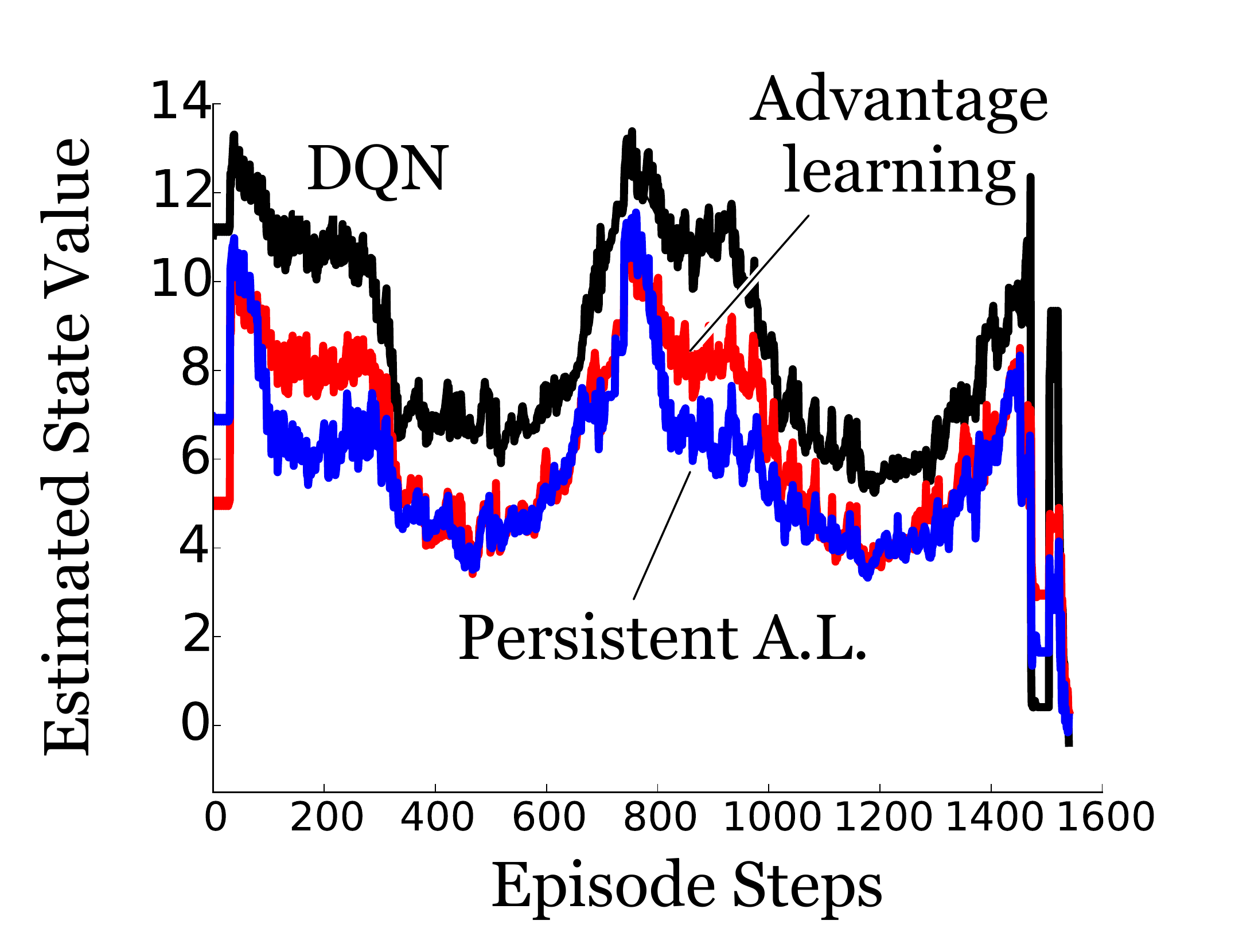}
}
\caption{Action gaps (\textbf{left}) and value functions (\textbf{right}) for a single episode of \textsc{Space Invaders} (Original DQN setting). 
Our operators yield markedly increased action gaps and lower values.
\label{fig:atari_probe_results}}
\end{figure}

The value functions and action gaps resulting from this experiment\footnote{Videos: https://youtu.be/wDfUnMY3vF8} are depicted in Figure \ref{fig:atari_probe_results}. As expected, the action gaps are significantly greater for both of our operators, in comparison to the action gaps produced by DQN. Furthermore, the value estimates are themselves lower, and correspond to more realistic estimates of the true value function.
In their experiments, van Hasselt et al.
observed a similar effect on the value estimates when replacing the Bellman updates with \emph{Double Q-Learning} updates, one of many solutions 
recently proposed to mitigate the negative impact of statistical bias in value function estimation \cite{vanhasselt10double,azar11speedy,lee13biascorrected}.
This bias is positive and is a consequence of the max term in the Bellman operator. We hypothesize that the lower value estimates observed in Figure \ref{fig:atari_probe_results} are also a consequence of bias reduction. Specifically, increased action gaps are consistent with a bias reduction: it is easily shown that the value estimation bias is strongest when Q-values are close to each other. 
If our hypothesis holds true, the third benefit of increasing the action gap is thus to \emph{mitigate the statistical bias of Q-value estimates}. 

\section{Open Questions}

\textbf{Weaker Conditions for Optimality.} At the core of our results lies the redefinition of Q-values in order to facilitate approximate value estimation. Theorem \ref{thm:optimality_achieving_operators} and our empirical results indicate that there are many practical operators which do not preserve suboptimal Q-values. Naturally, preserving the optimal value function $V$ is itself unnecessary, as long as the iterates converge to a Q-function $\tilde Q$ for which $\argmax\nolimits_{a} \tilde Q(x,a) = \pi^*(x)$. It may well be that even weaker conditions for optimality exist than those required by Theorem \ref{thm:optimality_achieving_operators}. At the present, however, our proof technique does not appear to extend to this case.  

\noindent \textbf{Statistical Efficiency of New Operators.} Advantage learning (as given by our redefinition) may be viewed as a generalization of the consistent Bellman operator when $P(\cdot \cbar x, a)$ is unknown or irrelevant. In this light, we ask: is there a probabilistic interpretation to advantage learning? We further wonder about the \emph{statistical efficiency} of the consistent Bellman operator: is it ever less efficient than the usual Bellman operator, when considering the probability of misclassifying the optimal action? Both of these answers might shed some light on the differences in performance observed in our experiments. 

\noindent \textbf{Maximally Efficient Operator.} Having revealed the existence of a broad family of optimality-preserving operators, we may now wonder which of these operators, if any, should be preferred to the Bellman operator. Clearly, there are trivial MDPs on which any optimality-preserving operator performs equally well. However, we may ask whether there is, for a given MDP, a ``maximally efficient'' optimality-preserving operator; and whether a learning agent can benefit from simultaneously searching for this operator while estimating a value function. 

\section{Concluding Remarks}

We presented in this paper a family of optimality-preserving operators, of which the consistent Bellman operator is a distinguished member. At the center of our pursuits lay the desire to increase the action gap; we showed through experiments that this gap plays a central role in the performance of greedy policies over approximate value functions, and how significantly increased performance could be obtained by a simple modification of the Bellman operator.
We believe our work highlights the inadequacy of the classical Q-function at producing reliable policies in practice, calls into question the traditional policy-value relationship in value-based reinforcement learning, and illustrates how revisiting the concept of value itself can be fruitful. 

\section{Acknowledgments}

The authors thank Michael Bowling, Csaba Szepesv\'ari, Craig Boutilier, Dale Schuurmans, Marty Zinkevich, Lihong Li, Thomas Degris, and Joseph Modayil for useful discussions, as well as the anonymous reviewers for their excellent feedback.

\bibliographystyle{aaai}
\bibliography{increasing-action-gap}

\newpage

\section{Appendix}

This appendix is divided into three sections. In the first section we present the proofs of our theoretical results. In the second we provide experimental details and additional results for the Bicycle domain. In the final section we provide details of our experiments on the Arcade Learning Environment, including results on 60 games.

\section{Theoretical Results}\label{sec:theory}

\begin{lem}\label{thm:convergence_of_gap_increasing_operators}
Let $Q \in \sQ$ and $\pi^Q$ be the policy greedy with respect to $Q$. Let $\cT'$ be an operator with the properties that, for all $x \in \cX, a \in \cA$, 
\begin{enumerate}
    \item{$\cT' Q(x,a) \le \cT Q(x,a)$, and}
    \item{$\cT'Q(x,\pi^Q(x)) = \cT Q(x,\pi^Q(x))$.} 
\end{enumerate}
Consider the sequence $Q_{k+1} := \cT' Q_k$ with $Q_0 \in \sQ$, and let $V_k(x) := \max_a Q_k(x, a)$.
Then the sequence $( V_k : k \in \bN)$ converges, and furthermore, for all $x \in \cX$,
\begin{equation*}
\lim_{k \to \infty} V_k(x) \le V^*(x).
\end{equation*}
\end{lem}
\begin{proof}
By Condition 1, we have that
\begin{align*}
\limsup_{k \to \infty} Q_k(x,a) &= \limsup_{k \to \infty} (\cT')^k Q_0(x,a) \\
&\le \limsup_{k \to \infty} \cT^k Q_0(x,a) \\
&= Q^*(x,a),
\end{align*}
since $\cT$ has a unique fixed point. From this we deduce the second claim. Now, for a given $x \in \cX$, let $a_k := \pi_k(x) := \argmax_a Q_k(x,a)$ and $P_k := P(\cdot \cbar x, a_k)$. We have 
\begin{align}
V_{k+1}(x) &\ge Q_{k+1}(x, a_k) = \cT' Q_k(x, a_k) \nonumber \\
&= \cT Q_k(x, a_k) \nonumber \\
&= \cT Q_{k-1} (x, a_k) + \gamma \expects_{P_k} \left [ V_k(x') - V_{k-1}(x') \right ]\nonumber \\
&\ge \cT'Q_{k-1}(x, a_k) + \gamma \expects_{P_k} \left [ V_k(x') - V_{k-1}(x') \right ] \nonumber \\
&= V_k(x) + \gamma \expects_{P_k} \left [ V_k(x') - V_{k-1}(x') \right ] \nonumber ,
\end{align}
where in the second line we used Condition 2 of the lemma, and in the third the definition of $\cT$ applied to $Q_k$. Thus we have
\begin{equation*}
V_{k+1}(x) - V_k(x) \ge \gamma \expects_{P_k} \left [ V_k(x') - V_{k-1}(x') \right ],
\end{equation*}
and by induction
\begin{equation}
V_{k+1}(x) - V_k(x) \ge \gamma^k \expects_{P_{1:k}} \left [ V_1(x') - V_0(x') \right ] \label{eqn:bound_on_delta_V},
\end{equation}
where $P_{1:k} := P_k P_{k-1} \dots P_1$ is the $k$-step transition kernel at $x$ derived from the nonstationary policy $\pi_k \pi_{k-1} \dots \pi_1$. 
Let $\tiV(x) := \limsup_{k \to \infty} V_k(x)$. We now show that $\liminf_{k \to \infty} V_k(x) = \tiV(x)$ also. First note that Conditions 1 and 2, together with the boundedness of $V_0$, ensure that $V_1$ is also bounded and thus $\infnorm{V_1 - V_0} < \infty$. 
By definition, for any $\delta > 0$ and $n \in \bN$, $\exists k \ge n$ such that $V_k(x) > \tiV(x) - \delta$. Since $P_{1:k}$ is a nonexpansion in $\infty$-norm, we have 
\begin{align*}
V_{k+1}(x) - V_k(x) &\ge -\gamma^k \infnorm{V_1 - V_0} \\
&\ge -\gamma^n \infnorm{V_1 - V_0} =: -\epsilon, \\
\end{align*}
and for all $t \in \bN$,
\begin{equation*}
V_{k+t}(x) - V_k(x) \ge - \sum_{i=0}^{t-1} \gamma^i \epsilon \ge \frac{-\epsilon}{1 - \gamma},
\end{equation*}
such that
\begin{equation*}
\inf_{t \in \bN} V_{k+t}(x) \ge \tiV(x) - \delta - \frac{\epsilon}{1 - \gamma}.
\end{equation*}
It follows that for any $x \in \cX$ and $\delta' > 0$, we can choose an $n \in \bN$ to make $\epsilon$ small enough such that for all $k \ge n$, $V_k(x) > \tiV(x) - \delta'$. Hence 
\begin{equation*}
\liminf_{k \to \infty} V_k(x) = \tiV(x),
\end{equation*}
and thus $V_k(x)$ converges. 
\end{proof}

\begin{lem}\label{lem:V_k_is_bounded}
Let $\cT'$ be an operator satisfying the conditions of Lemma \ref{thm:convergence_of_gap_increasing_operators}, and let $\Rmax := \max\nolimits_{x,a} R(x,a)$. Then for all $x \in \cX$ and all $k \in \bN$,
\begin{equation}\label{eqn:V_k_bounded}
|V_k(x)| \le \frac{1}{1 - \gamma} \Big [ 2 \infnorm{V_0} + \Rmax \Big ] .
\end{equation}
\end{lem}
\begin{proof}
Following the derivation of Lemma \ref{thm:convergence_of_gap_increasing_operators}, we have
\begin{align*}
V_{k+1}(x) - V_0(x) &\ge -\sum_{i=1}^k \gamma^i \infnorm{V_1 - V_0} \\
&\ge \frac{-1}{1 - \gamma} \infnorm{V_1 - V_0}. 
\end{align*}
By the same derivation, for $a_0 := \argmax_a Q_0(x, a)$ we have
\begin{equation*}
V_1(x) \ge \cT Q_0(x, a_0).
\end{equation*}
But then
\begin{equation*}
V_1(x) - V_0(x) \ge R(x, a_0) + \gamma \expects_{P_0} V_0(x') - V_0(x),
\end{equation*}
from which the lower bound follows. Now let $P_k$ be defined as in the proof of Lemma \ref{thm:convergence_of_gap_increasing_operators}, and assume the upper bound of (\ref{eqn:V_k_bounded}) holds up to $k \in \bN$. Then 
\begin{align*}
V_{k+1}(x) &= \max\nolimits_a Q_{k+1}(x, a) = \max\nolimits_a \cT' Q_k(x,a) \\
&\le \max\nolimits_a \cT Q_k(x,a) \\
&= \max\nolimits_a \left [ R(x,a) + \gamma \expects_{P_k} V_k(x') \right ] \\
&\le \Rmax + \gamma \infnorm{V_k} \\
&\le \Rmax + \frac{\gamma}{1 - \gamma} \left [ 2 \infnorm{V_0} + \Rmax \right ] \\
&\le \frac{1}{1 - \gamma} \left [ 2 \infnorm{V_0} + \Rmax \right ],
\end{align*}
and combined with the fact that (\ref{eqn:V_k_bounded}) holds for $k = 0$ this proves the upper bound. 
\end{proof}

\begin{thm}
Let $\cT$ be the Bellman operator ((1) in the main text). Let $\cT'$ be an operator with the property that there exists an $\alpha \in [0, 1)$ such that for all $Q \in \sQ$, $x \in \cX, a \in \cA$, and letting $V(x) := \max\nolimits_b Q(x,b)$,
\begin{enumerate}
\item{$\cT' Q(x, a) \le \cT Q(x,a)$, and}
\item{$\cT' Q(x, a) \ge \cT Q(x,a) - \alpha \left [ V(x) - Q(x,a) \right ]$.}
\end{enumerate}
Consider the sequence $Q_{k+1} := \cT' Q_k$ with $Q_0 \in \sQ$, and $V_k(x) := \max_a Q_k(x,a)$. Then $\cT'$ is \emph{optimality-preserving}: for all $x \in \cX$, $(V_k(x) : k \in \bN)$ converges, 
\begin{equation*}
\lim_{k \to \infty} V_k(x) = V^*(x), 
\end{equation*}
and
\begin{equation*}
Q^*(x,a) < V^*(x) \, \implies \, \limsup_{k \to \infty} Q_k(x,a) < V^*(x).
\end{equation*}
Furthermore, $\cT'$ is also \emph{gap-increasing}:
\begin{equation*}
\liminf_{k \to \infty} \big [ V_k(x) - Q_k(x,a) \big ] \ge V^*(x) - Q^*(x,a) .
\end{equation*}
\end{thm}

\begin{proof}
Note that these conditions imply the conditions of Lemma \ref{thm:convergence_of_gap_increasing_operators}. Thus for all $x \in \cX$, $(V_k(x) : k \in \bN)$ converges to the limit $\tiV(x) \le V^*(x)$. Now let $\tiQ(x,a) := \limsup_k Q_k(x,a)$. We have
\begin{align}
\tiQ(x,a) &= \limsup_{k \to \infty} \cT' Q_k(x,a) \nonumber \\
&\le \limsup_{k \to \infty} \cT Q_k(x, a) \nonumber \\
&= \limsup_{k \to \infty} \left [ R(x,a) + \gamma \expects_P \max_{b \in \cA} Q_k(x', b) \right ] \nonumber \\
&\le R(x,a) + \gamma \expects_P \limsup_{k \to \infty} \max_{b \in \cA} Q_k(x', b) \label{eqn:limsup_expectation_swap} \\
&= R(x,a) + \gamma \expects_P \max_b \limsup_{k \to \infty} Q_k(x', b) \label{eqn:max_limsup_commute} \\
&= \cT \tiQ(x, a), \label{eqn:q_smaller_than_tq} 
\end{align}
where in (\ref{eqn:limsup_expectation_swap}) we used Jensen's inequality, and (\ref{eqn:max_limsup_commute}) follows from the commutativity of $\max$ and $\limsup$. Now
\begin{align}
Q_{k+1}(x,a) &= \cT' Q_k(x,a) \nonumber \\
&\ge \cT Q_k(x,a) - \alpha\left [ V_k(x) - Q_k(x,a) \right ] \nonumber \\
&= R(x,a) + \gamma \expects_P V_k(x') - \alpha V_k(x) \oneem + \nonumber \\
& \qquad \alpha Q_k(x, a). \label{eqn:Q_k_inequality}
\end{align}
Now, by Lemma \ref{thm:convergence_of_gap_increasing_operators} $V_k(x)$ converges to $\tiV(x)$. Furthermore, using Lemma \ref{lem:V_k_is_bounded} and Lebesgue's dominated convergence theorem, we have
\begin{equation}
\lim_{k \to \infty} \expects_P V_k(x') = \expects_P \tiV(x') \label{eqn:dominated_convergence}.
\end{equation}
We now take the $\limsup$ of both sides of (\ref{eqn:Q_k_inequality}), which Lemma \ref{lem:V_k_is_bounded} guarantees exists, and obtain
\begin{align*}
\tiQ(x,a) &\ge R(x,a) + \gamma \expects_P \tiV(x') - \alpha \tiV(x) + \alpha \tiQ(x,a) \\
&= \cT \tiQ(x,a) - \alpha \tiV(x) + \alpha \tiQ(x,a).
\end{align*}
Thus
\begin{align*}
\tiQ (x,a) &\ge \frac{1}{1 - \alpha} \left [ \cT \tiQ (x,a) - \alpha \tiV(x) \right ], \text{ and }\\
\tiV (x) &\ge \frac{1}{1 - \alpha} \left [ \max_{a \in \cA} \cT \tiQ (x,a) - \alpha \tiV(x) \right ] \\
\tiV (x) &\ge \max_{a \in \cA} \cT \tiQ (x,a) .
\end{align*}
Combining the above with (\ref{eqn:q_smaller_than_tq}), we deduce that
\begin{equation*}
\tiV(x) = \max_{a \in \cA} \cT \tiQ(x,a) = \max_{a \in \cA} \left [ R(x,a) + \gamma \expects_P \tiV(x') \right ]
\end{equation*}
and, by uniqueness of the fixed point of the Bellman operator over $\sV$, it must be that $\tiV = V^*$.

Now suppose that for some $x \in \cX$, $\tilde a \in \cA$, we have
\begin{equation*}
Q^*(x, \tilde a) < V^*(x).
\end{equation*}
By Condition 1 
\begin{align*}
Q_k(x, \tilde a) &= \cT' Q_{k-1}(x, \tilde a) \\
&\le \cT Q_{k-1}(x, \tilde a) \\
&=\cT Q^*(x, \tilde a) - \gamma \expects_{P_{\tia}} \left [ V^*(x') - V_{k-1}(x') \right ] \\
&=Q^*(x, \tia) - \gamma \expects_{P_{\tia}} \left [ V^*(x') - V_{k-1}(x') \right ],
\end{align*}
where $P_{\tia} := P(\cdot \cbar x, \tia)$. Using (\ref{eqn:dominated_convergence}) we take the $\limsup$ on both sides and find that
\begin{align*}
\limsup_{k \to \infty} Q_k(x, \tia) &\le Q^*(x, \tia) - \gamma \expects_{P_{\tia}} \left [ V^*(x') - \tiV(x') \right ] \\
&= Q^*(x,\tia) \\
& < V^*(x) .
\end{align*}
We conclude that
\begin{equation*}
Q^*(x,a) < V^*(x) \, \implies \, \limsup_{k \to \infty} Q_k(x,a) < V^*(x).
\end{equation*}
Hence, $\cT'$ is optimality-preserving. To prove that $\cT'$ is gap-increasing, observe that the statement 
\begin{equation*}
\liminf_{k \to \infty} \big [ V_k(x) - Q_k(x,a) \big ] \ge V^*(x) - Q^*(x,a)
\end{equation*}
is now equivalent to
\begin{equation}\label{eqn:gap_increasing_reformulated}
\limsup_{k \to \infty} Q_k(x,a) \le Q^*(x,a)
\end{equation}
since $\lim_k V_k(x) = V^*(x)$. But we know (\ref{eqn:gap_increasing_reformulated}) to be true from Condition 1 (see the proof of Lemma \ref{thm:convergence_of_gap_increasing_operators}). 
\end{proof}

\begin{cor}
The consistent Bellman operator $\cTc$ ((5) in the main text) and consistent Q-value interpolation Bellman operator $\cTcq$ ((9) in the main text) are optimality-preserving and gap-increasing. 
\end{cor}

\section{Experimental Details: Bicycle}\label{sec:bicycle}

We used the bicycle simulator described by \citet{randlov98learning} with a reward function which encourages driving towards the goal. Recall that Randlov and Alstrom's reward function is
\begin{equation*}
R(x,a) := \left \{ \begin{array}{ll}
    -1 & \text{if bicycle falls} \\
    0.01 & \text{if goal is reached} \\
    (4 - \psi^2) \times 0.00004 & \text{otherwise}
    \end{array} \right .
\end{equation*}
As noted by Randlov and Alstrom themselves, this reward function is unsuitable for value iteration methods, since it rewards driving away from the goal. Instead we use the following related reward function 
\begin{equation*}
R(x,a) := \left \{ \begin{array}{ll}
    -c & \text{if fallen} \\
    1.0 & \text{if goal reached } \\
    (\pi^2/4 - \psi^2 - 1) \times 0.001 & \text{otherwise}
    \end{array} \right . 
\end{equation*}
with $c := (\frac{3}{4} \pi^2 - 1) \times 0.001$ the largest negative reward achievable by the agent. Empirically, we found this reward function easier to work with, while our results remained qualitatively similar for similar reward functions. We further use a discount factor of $\gamma = 0.99$.

We consider two sample-based operators on $\sQ_{\cZ, \cA}$, the space of Q-functions over representative states. The sample-based Q-value interpolation Bellman operator is defined as
\begin{equation*}
\cTq Q(z,a) := R(z,a) + \gamma \frac{1}{k} \sum_{i=1}^k \max_{b \in \cA} Q(x_i', b),
\end{equation*}
with $k \in \bN$ and $x_i' \sim P(\cdot \cbar z, a)$. The sample-based consistent Q-value interpolation Bellman operator $\cTcq$ is similarly defined by sampling $x'$ from $P$:
\begin{align*}
\cTq' Q(z,a) &:= R(z,a) + \\
& \hspace{-3.5em} \frac{\gamma}{k} \sum_{i=1}^k \max_{b \in \cA} \Big [ Q(x', b) - A(z \cbar x') \left( Q(z,b) - Q(z,a) \right ) \Big ] \\
\cTcq Q(z,a) &:= \min \left \{ \cTq Q(z,a), \cTq' Q(z,a) \right \} .
\end{align*}
In both cases, we use Q-value interpolation to define a Q-function over $\cX$: 
\begin{equation*}
Q(x,a) := \expects_{z \sim A(\cdot \cbar x)} Q(z, a).
\end{equation*}
For each operator $\cT'$, we computed a sequence of Q-functions $Q_k \in \sQ_{\cZ, \cA}$ using an averaging form of value iteration:
\begin{equation*}
Q_{k+1}(z,a) = (1 - \eta) Q_k(z,a) + \eta \cT' Q_k(z,a),
\end{equation*}
applied simultaneously to all $z \in \cZ$ and $a \in \cA$. We chose this averaging version because it led to faster convergence, and lets us take $k = 1$ in the definition of both operators. From a parameter sweep we found $\eta = 0.1$ to be a suitable step-size.

Our multilinear grid was defined over the six state variables. As done elsewhere in the literature, we defined our grid over the following bounded variables: 
\begin{align*}
\omega & \in \left [ - \frac{4}{9} \pi, \frac{4}{9} \pi \right ], \\
\dot \omega & \in [ -2, 2 ], \\ 
\theta & \in \left [ - \frac{\pi}{15}, \frac{\pi}{15} \right ], \\
\dot \theta & \in \left [ -0.5, 0.5 \right ], \\
\psi & \in [ -\pi, \pi ], \\
d & \in [ 10, 1200 ] .
\end{align*}
Values outside of these ranges were accordingly set to the range's minimum or maximum.

For completeness, Figure \ref{fig:bicycle-order-8} compares the performance of the Bellman and consistent Bellman operators, as well as advantage learning and persistent advantage learning (with $\alpha = 0.1$), on $8 \times \dots \times 8$ and $10 \times \dots \times 10$ grids. Here, the usual Bellman operator is unable to find a solution to the goal, while the consistent Bellman operator successfully does so. The two other operators also achieve superior performance compared to Bellman operator, although appear slightly more unstable in the smaller grid setting. 

\begin{figure}
\centering{
\includegraphics[width=1.6in]{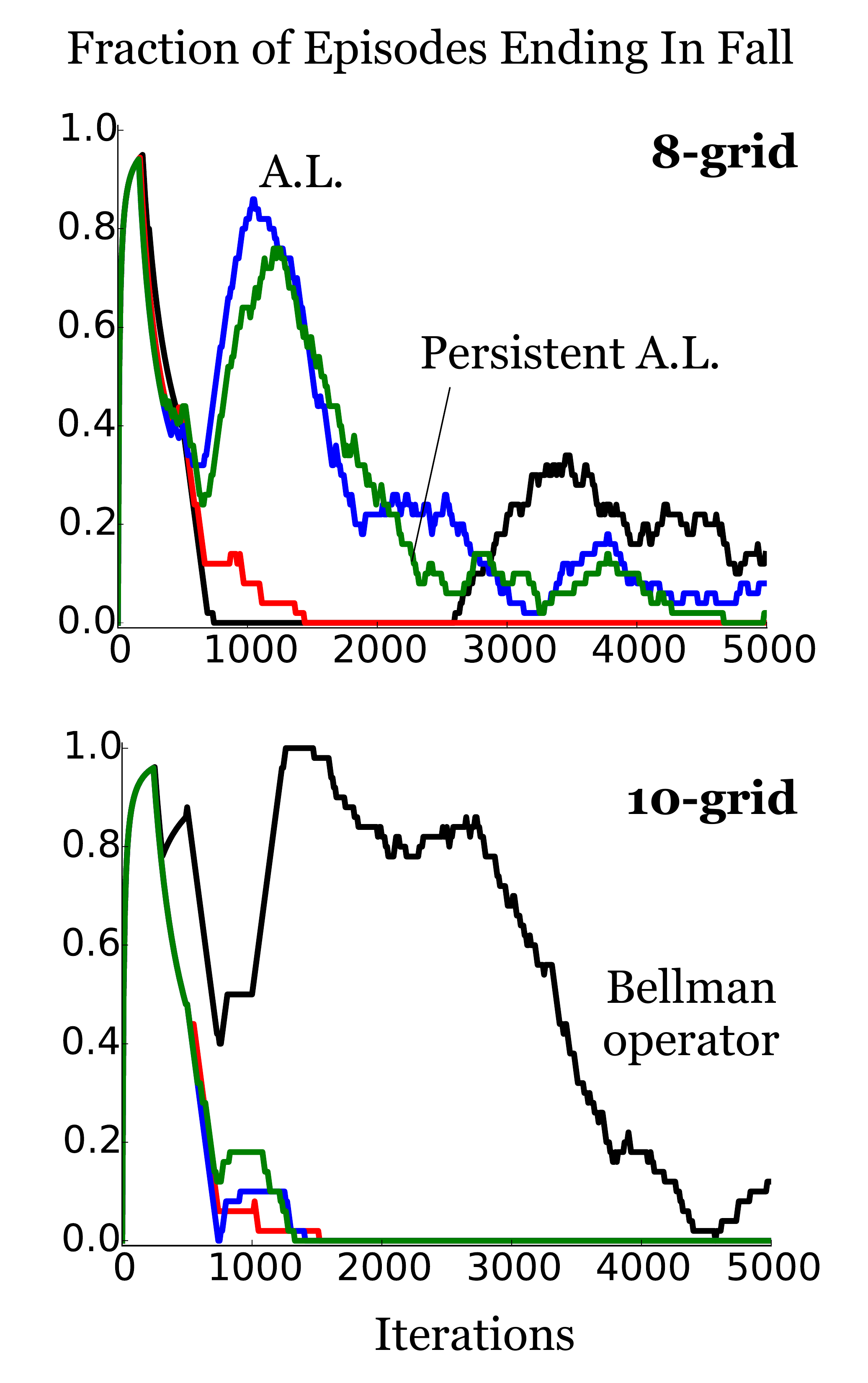}
\includegraphics[width=1.6in]{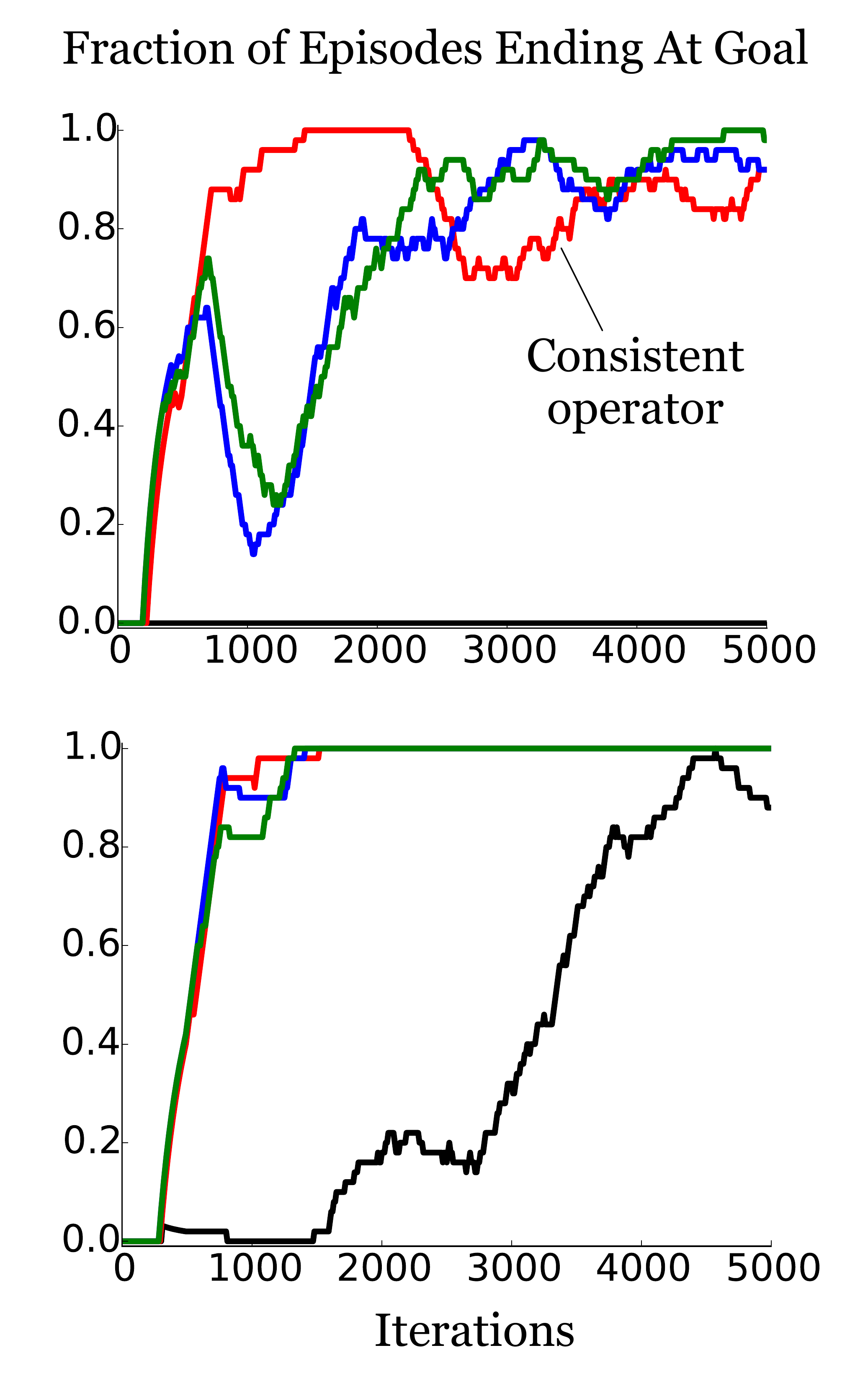}
}
\caption{\textbf{Top.} Falling and goal-reaching frequency for greedy policies derived from value iteration on a $8 \times \dots \times 8$ grid. \textbf{Bottom.} The same, for a $10 \times \dots \times 10$ grid.\label{fig:bicycle-order-8}}
\end{figure}

\section{Experimental Details: ALE}\label{sec:atari}

We omit details of the DQN architecture, which are provided in \citet{mnih15human}. A \emph{frame} is a single emulation step within the ALE, while a \emph{time step} consists of four consecutive frames which are treated atomically by the agent. 

Our first Atari 2600 experiment (\emph{Stochastic Minimal} setting) used stochastic controls, which operate as follows: at each frame (not time step), the environment \emph{accepts} the agent's action with probability $1-p$, or \emph{rejects} it with probability $p$. If an action is rejected, the previous frame's action is repeated. In our setting, the agent selects a new action every four frames; in this situation, the stochastic controls approximate a form of reaction delay. This particular setting is part of the latest Arcade Learning Environment. For our experiments we use the ALE 0.5 standard value of $p=0.25$, and trained agents for 100 million frames.

Our second Atari 2600 experiment (\emph{Original DQN} setting) was averaged over three different trials, ran for 200 million frames (instead of 100 million), defined a lost life as a termination signal, and did not use stochastic controls. This matches the experimental setting of \citet{mnih15human}. A full table of our results is provided in Table \ref{table:full_atari_results}. 
\begin{table*}[p]
\begin{center}
\vspace{-4em}
\include{atari_table}

\caption{Highest performance achieved by each of our operators. For each game, the score of the best operator is highlighted. Games with a $\dagger$ were not used by \citet{bellemare13arcade}. See Section 4 of the main text for more details.\label{table:full_atari_results}}
\end{center}
\end{table*}

Our last experiment took place in the Original DQN setting. We generated a trajectory from a trained DQN agent playing an $\epsilon$-greedy policy with $\epsilon = 0.05$. The full trajectory (up to the end of the episode) was recorded in this way. We then queried the value functions of the trained agents, including the DQN used to generate the trajectory, in order to generate Figure 4 of the main text. For clarity we report action gaps averaged according to a rolling window of length 50. 

Out of the 60 games for which we report results, 5 are new when compared to the table of results provided by \citet{bellemare13arcade}. These five games are identified with a $\dagger$ in Table \ref{table:full_atari_results}.

\subsection{DQN Implementation Details}

Recall that DQN maintains two networks in parallel: a \emph{policy} network, which is used to select actions and is updated at every time step, and a \emph{target} network. The target network is used to compute the error term $\Delta Q$, and is only updated every 10,000 time steps \citep{mnih15human}. In our experiments we also used this target network to compute the $\Delta_{\textsc{al}} Q$ and $\Delta_{\textsc{pal}} Q$, including the added correction term. Our operators performed worse when the correction term was instead computed from the policy network. 

\subsection{Parameter Selection}

\begin{figure*}[ht!]
\centering{
\includegraphics[width=5.6in]{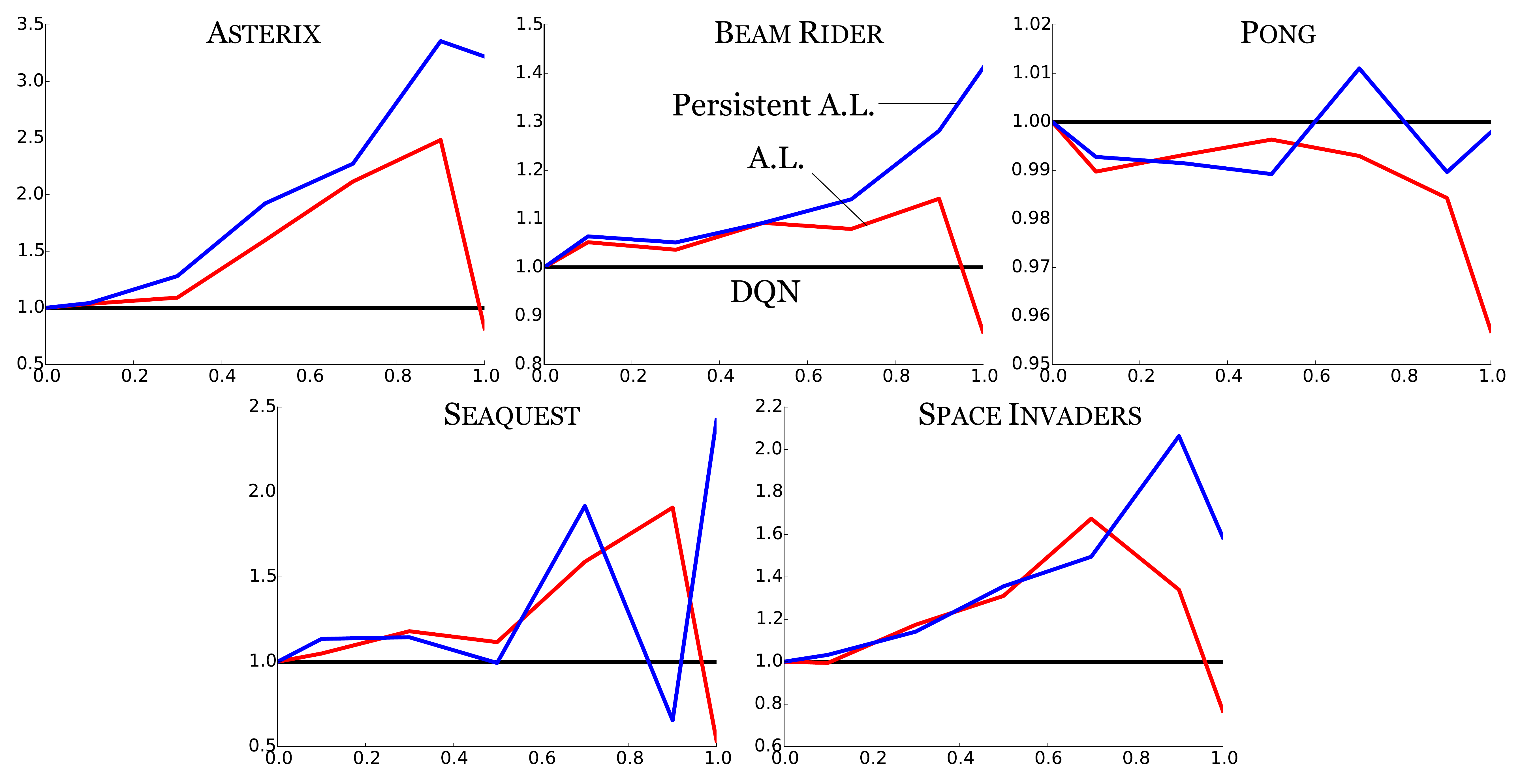}
}
\caption{Performance of trained agents in function of the $\alpha$ parameter. Note that $\alpha = 1.0$ does not satisfy our theorem's conditions. We attribute the odd performance of Seaquest agents using Persistent Advantage Learning with $\alpha = 0.9$ to a statistical issue.\label{fig:alpha_parameter}}
\end{figure*}

We used five training games (\gamename{Asterix}, \gamename{Beam Rider}, \gamename{Pong}, \gamename{Seaquest}, \gamename{Space Invaders}) to select the $\alpha$ parameter for both of our operators. Specifically, we trained agents using our second experimental setup with parameters $\alpha \in \left \{ 0.0, 0.1, 0.3, 0.5, 0.7, 0.9, 1.0 \right \}$, evaluated them according to the highest score achieved, and manually selected the $\alpha$ value which seemed to achieve the best performance. Note that $\alpha = 0.0$ corresponds to DQN in both cases. Figure \ref{fig:alpha_parameter} depicts the results of this parameter sweep. 

\end{document}

%% file: atari_table.tex
\begin{tabular}{|r|r|r|r|}
\hline
Game & Bellman & Advantage Learning & Persistent A.L. \\ 
\hline
\textsc{Asterix} & 6074.98 & 12852.08 & \textbf{\textcolor{blue}{19564.90}} \\
\hline
\textsc{Beam Rider} & 9316.10 & 10054.58 & \textbf{\textcolor{blue}{13145.34}} \\
\hline
\textsc{Pong} & \textbf{\textcolor{blue}{19.80}} & 19.66 & 19.76 \\
\hline
\textsc{Seaquest} & 5458.17 & 8670.50 & \textbf{\textcolor{blue}{13230.74}} \\
\hline
\textsc{Space Invaders} & 2067.19 & \textbf{\textcolor{blue}{3460.79}} & 3277.59 \\
\hline
\hline
\textsc{Alien} & 3154.67 & 4990.91 & \textbf{\textcolor{blue}{5699.81}} \\
\hline
\textsc{Amidar} & 969.88 & \textbf{\textcolor{blue}{1557.43}} & 1451.65 \\
\hline
\textsc{Assault} & \textbf{\textcolor{blue}{4573.67}} & 3661.51 & 3304.33 \\
\hline
\textsc{Asteroids} & 1827.97 & \textbf{\textcolor{blue}{1924.42}} & 1673.52 \\
\hline
\textsc{Atlantis} & 636657.62 & 553591.67 & \textbf{\textcolor{blue}{1465250.00}} \\
\hline
\textsc{Bank Heist} & 511.00 & 633.63 & \textbf{\textcolor{blue}{874.99}} \\
\hline
\textsc{Battle Zone} & 28082.91 & 28789.29 & \textbf{\textcolor{blue}{34583.07}} \\
\hline
\textsc{Berzerk} & 667.61 & 747.26 & \textbf{\textcolor{blue}{1328.25}} \\
\hline
\textsc{Bowling} & \textbf{\textcolor{blue}{74.62}} & 57.41 & 71.59 \\
\hline
\textsc{Boxing} & 88.66 & 93.94 & \textbf{\textcolor{blue}{94.30}} \\
\hline
\textsc{Breakout} & 378.69 & 425.32 & \textbf{\textcolor{blue}{431.89}} \\
\hline
\textsc{Carnival} & \textbf{\textcolor{blue}{5238.14}} & 5111.40 & 4679.93 \\
\hline
\textsc{Centipede} & \textbf{\textcolor{blue}{5719.11}} & 4225.18 & 4539.55 \\
\hline
\textsc{Chopper Command} & \textbf{\textcolor{blue}{8195.88}} & 5431.36 & 5734.93 \\
\hline
\textsc{Crazy Climber} & 114105.56 & 123410.71 & \textbf{\textcolor{blue}{130002.71}} \\
\hline
\textsc{Defender$^\dagger$} & 16746.68 & 30643.59 & \textbf{\textcolor{blue}{32038.93}} \\
\hline
\textsc{Demon Attack} & 23212.19 & 27153.48 & \textbf{\textcolor{blue}{70908.17}} \\
\hline
\textsc{Double Dunk} & -6.23 & \textbf{\textcolor{blue}{-0.15}} & -2.51 \\
\hline
\textsc{Elevator Action} & 26675.00 & 27088.89 & \textbf{\textcolor{blue}{29100.00}} \\
\hline
\textsc{Enduro} & 776.14 & 1252.70 & \textbf{\textcolor{blue}{1343.10}} \\
\hline
\textsc{Fishing Derby} & 11.65 & 21.32 & \textbf{\textcolor{blue}{28.13}} \\
\hline
\textsc{Freeway} & 31.14 & 31.72 & \textbf{\textcolor{blue}{32.30}} \\
\hline
\textsc{Frostbite} & 1485.42 & 2305.82 & \textbf{\textcolor{blue}{3248.96}} \\
\hline
\textsc{Gopher} & 8479.98 & \textbf{\textcolor{blue}{11912.68}} & 10611.81 \\
\hline
\textsc{Gravitar} & \textbf{\textcolor{blue}{448.74}} & 417.65 & 446.92 \\
\hline
\textsc{H.E.R.O.} & 18490.97 & \textbf{\textcolor{blue}{24788.86}} & 24175.79 \\
\hline
\textsc{Ice Hockey} & -2.13 & -1.24 & \textbf{\textcolor{blue}{-0.25}} \\
\hline
\textsc{James Bond} & \textbf{\textcolor{blue}{867.84}} & 848.46 & 772.09 \\
\hline
\textsc{Kangaroo} & 9157.98 & 10809.16 & \textbf{\textcolor{blue}{11478.46}} \\
\hline
\textsc{Krull} & 8500.48 & \textbf{\textcolor{blue}{9548.92}} & 8689.81 \\
\hline
\textsc{Kung-Fu Master} & 25977.53 & 32182.99 & \textbf{\textcolor{blue}{34650.91}} \\
\hline
\textsc{Montezuma's Revenge} & 0.64 & 0.42 & \textbf{\textcolor{blue}{1.72}} \\
\hline
\textsc{Ms. Pac-Man} & 3081.29 & \textbf{\textcolor{blue}{4065.80}} & 3917.55 \\
\hline
\textsc{Name This Game} & 8585.03 & \textbf{\textcolor{blue}{11025.26}} & 10431.33 \\
\hline
\textsc{Phoenix$^\dagger$} & 14278.95 & \textbf{\textcolor{blue}{22038.27}} & 14495.56 \\
\hline
\textsc{Pitfall!$^\dagger$} & \textbf{\textcolor{blue}{0.00}} & \textbf{\textcolor{blue}{0.00}} & \textbf{\textcolor{blue}{0.00}} \\
\hline
\textsc{Pooyan} & 4736.79 & 4801.27 & \textbf{\textcolor{blue}{5858.84}} \\
\hline
\textsc{Private Eye} & 957.83 & \textbf{\textcolor{blue}{5276.16}} & 339.15 \\
\hline
\textsc{Q*Bert} & 10840.83 & \textbf{\textcolor{blue}{14368.03}} & 14254.78 \\
\hline
\textsc{River Raid} & 7315.20 & 10585.12 & \textbf{\textcolor{blue}{12813.27}} \\
\hline
\textsc{Road Runner} & 38042.07 & \textbf{\textcolor{blue}{52351.23}} & 37856.16 \\
\hline
\textsc{Robotank} & 61.97 & 69.31 & \textbf{\textcolor{blue}{70.53}} \\
\hline
\textsc{Skiing} & -13049.42 & -13264.51 & \textbf{\textcolor{blue}{-12173.35}} \\
\hline
\textsc{Solaris$^\dagger$} & 4638.85 & \textbf{\textcolor{blue}{4785.16}} & 3274.70 \\
\hline
\textsc{Star Gunner} & 55558.27 & 61353.59 & \textbf{\textcolor{blue}{61521.87}} \\
\hline
\textsc{Surround} & -5.79 & -4.15 & \textbf{\textcolor{blue}{0.72}} \\
\hline
\textsc{Tennis} & \textbf{\textcolor{blue}{0.00}} & \textbf{\textcolor{blue}{0.00}} & \textbf{\textcolor{blue}{0.00}} \\
\hline
\textsc{Time Pilot} & 5788.96 & \textbf{\textcolor{blue}{8969.12}} & 8749.26 \\
\hline
\textsc{Tutankham} & 200.17 & \textbf{\textcolor{blue}{245.22}} & 197.33 \\
\hline
\textsc{Up and Down} & 12831.57 & \textbf{\textcolor{blue}{13909.74}} & 13542.07 \\
\hline
\textsc{Venture} & \textbf{\textcolor{blue}{373.79}} & 198.69 & 243.75 \\
\hline
\textsc{Video Pinball} & \textbf{\textcolor{blue}{611840.72}} & 543504.00 & 542052.00 \\
\hline
\textsc{Wizard Of Wor} & 2410.47 & 9541.14 & \textbf{\textcolor{blue}{10254.01}} \\
\hline
\textsc{Yar's Revenge$^\dagger$} & 21440.45 & \textbf{\textcolor{blue}{24240.03}} & 17141.56 \\
\hline
\textsc{Zaxxon} & 6416.06 & \textbf{\textcolor{blue}{9129.61}} & 8155.60 \\
\hline
\hline
Times Best & 12 & 21 & 31 \\
\hline
\end{tabular}